\pdfoutput=1
\documentclass{article}

\usepackage[final,nonatbib]{neurips_2022}

\usepackage[utf8]{inputenc} 
\usepackage[T1]{fontenc}    
\usepackage{hyperref}       
\usepackage{url}            
\usepackage{booktabs}       
\usepackage{amsfonts}       
\usepackage{nicefrac}       
\usepackage{microtype}      

\usepackage{hyperref}

\usepackage{amssymb,amsmath}
\usepackage{amsthm}
\usepackage{mathtools}
\mathtoolsset{showonlyrefs}

\usepackage{color}
\usepackage{xcolor}

\usepackage{enumitem}
\usepackage{booktabs}

\usepackage{subcaption}
\usepackage{wrapfig}

\usepackage{float}

\def\cC{{\mathcal{C}}}
\def\cK{{\mathcal{K}}}
\def\cS{{\mathcal{S}}}
\def\cE{{\mathcal{E}}}

\def\RR{{\mathbb R}}

\def\prox{{\text{prox}}}
\def\proj{{\text{proj}}}
\def\sign{{\text{sign}}}
\def\ST{{\text{ST}}}

\DeclareMathOperator*{\argmin}{argmin}

\DeclareMathOperator*{\dom}{dom}
\DeclareMathOperator*{\relint}{relint}
\DeclareMathOperator*{\diag}{diag}

\newtheorem{theorem}{Theorem}

\newtheorem{corollary}{Corollary}
\newtheorem{definition}{Definition}

\def\ones{{\mathbf{1}}}

\def\training{{\mathtt{tr}}}
\def\validation{{\mathtt{val}}}

\def\dataMatrix{\Phi}

\hypersetup{
    colorlinks,
    linkcolor={black},
    citecolor={blue!50!black},
    urlcolor={blue!50!black}
}

\title{\Large{Efficient and Modular Implicit Differentiation}}

\author{%
    Mathieu Blondel,
    Quentin Berthet,
    Marco Cuturi\thanks{Work done while at Google Research, now at Apple and
    Owkin, respectively.},
    Roy Frostig, \\
    {\bf Stephan Hoyer,
    Felipe Llinares-L\'{o}pez,
    Fabian Pedregosa,
    Jean-Philippe Vert$^\ast$} \\
    \textnormal{Google Research}
}

\begin{document}

\maketitle

\begin{abstract}
Automatic differentiation (autodiff) has revolutionized machine learning.  It
allows to express complex computations by composing elementary ones in creative
ways and removes the burden of computing their derivatives by hand. More
recently, differentiation of optimization problem solutions has attracted
widespread attention with applications such as optimization layers, and in
bi-level problems such as hyper-parameter optimization and meta-learning.
However, so far, implicit differentiation remained difficult to use for
practitioners, as it often required case-by-case tedious mathematical
derivations and implementations. In this paper, we propose
automatic implicit differentiation, an efficient
and modular approach for implicit differentiation of optimization problems. In
our approach, the user defines directly in Python a function $F$ capturing the
optimality conditions of the problem to be differentiated. Once this is done, we
leverage autodiff of $F$ and the implicit function theorem to automatically
differentiate the optimization problem.  Our approach thus combines the benefits
of implicit differentiation and autodiff.  It is efficient as it can be added on
top of any state-of-the-art solver and modular as the optimality condition
specification is decoupled from the implicit differentiation mechanism.  We show
that seemingly simple principles allow to recover many existing implicit
differentiation methods and create new ones easily.  We demonstrate the ease of
formulating and solving bi-level optimization problems using our framework. We
also showcase an application to the sensitivity analysis of molecular dynamics.
\end{abstract}

\section{Introduction}

Automatic differentiation (autodiff) is now an inherent part of machine
learning software.  It allows to express complex computations by composing
elementary ones in creative ways and removes the tedious burden of computing
their derivatives by hand. In parallel, the differentiation of optimization
problem solutions has found many applications. A classical example is bi-level
optimization, which typically involves computing the derivatives of a nested
optimization problem in order to solve an outer one.
Examples of applications in machine learning include hyper-parameter
optimization \cite{chapelle_2002,seeger_2008,pedregosa2016hyperparameter,
franceschi_2017, bertrand_2020_implicit,bertrand_2021_journal},
neural networks \cite{lorraine_2020}, and
meta-learning \cite{franceschi_2018,meta_learning}. Another line of active
research involving differentiation of optimization problem solutions
are optimization layers
\cite{kim_2017,amos_2017,niculae_2017,djolonga_2017,
gould_2019}, which can be used to encourage structured outputs, and implicit
deep networks \cite{bai_2019,el_ghaoui_2019,fung2021fixed,geng_2021,ramzi_2021}, which have a smaller memory
footprint than backprop-trained networks.

Since optimization problem solutions typically do not enjoy an explicit formula
in terms of their inputs, autodiff cannot be used directly to differentiate
these functions.  In recent years, two main approaches have been developed to
circumvent this problem. The first one consists of unrolling the iterations of
an optimization algorithm and using the final iteration as a proxy for the
optimization problem solution \cite{wengert_1964, domke_2012, deledalle_2014,
franceschi_2018, ablin_2020}. 
This allows to \textbf{explicitly} construct a computational graph relating the
algorithm output to the inputs, on which autodiff can then be used
transparently.  
However, this
requires a reimplementation of the algorithm using the autodiff system, and not
all algorithms are necessarily autodiff friendly.
Moreover, forward-mode autodiff has time complexity that scales linearly with
the number of variables and reverse-mode autodiff has memory complexity
that scales linearly with the number of algorithm iterations.
In contrast, a second approach consists in \textbf{implicitly} relating
an optimization problem solution to its inputs using optimality conditions. 
In a machine learning context, such implicit differentiation has been used for
stationarity conditions \cite{bengio_2000,lorraine_2020}, KKT conditions
\cite{chapelle_2002, gould_2016, amos_2017,sparsemap,lp_sparsemap} and the
proximal gradient fixed point
\cite{niculae_2017,bertrand_2020_implicit,bertrand_2021_journal}. An advantage
of implicit differentiation is that a solver reimplementation is not needed,
allowing to build upon decades of state-of-the-art software.
Although implicit differentiation has a long history in numerical analysis
\cite{griewank_2008,bell_2008,krantz_2012,bonnans_2013}, so far, it remained
difficult to use for practitioners, as it required
a case-by-case tedious mathematical derivation and implementation. 
CasADi \cite{casadi} allows to differentiate various optimization and root
finding problem algorithms provided by the library. However, it does not allow
to easily add implicit differentiation on top of existing solvers from
optimality conditions expressed by the user, as we do.
A recent tutorial explains how to implement implicit differentiation in JAX
\cite{tutorial_implicit}.  However, the tutorial requires the user to take
care of low-level technical details and does not cover a large catalog of
optimality condition mappings as we do.
Other work \cite{agrawal_2019}
attempts to address this issue by adding implicit differentiation on top of
cvxpy \cite{cvxpy}. This works by reducing all convex optimization problems to a
conic program and using conic
programming's optimality conditions to derive an implicit differentiation
formula.
While this approach is very generic, solving a convex optimization problem
using a conic programming solver---an ADMM-based splitting conic solver
\cite{splitting_conic} in the case of cvxpy---is rarely state-of-the-art
for every problem instance. 

In this work, we ambition to achieve for optimization problem solutions what autodiff did for
computational graphs.
We propose \textbf{automatic implicit differentiation},
a simple approach to add implicit differentiation on top of any existing solver.
In this approach, the user defines directly in Python a mapping
function $F$ capturing the optimality conditions of the problem solved by the
algorithm.  Once this is done, we leverage autodiff of $F$ combined with
the implicit function theorem 
to automatically differentiate the optimization problem solution. 
Our approach is \textbf{generic}, yet it can exploit
the \textbf{efficiency} of state-of-the-art solvers. It therefore combines the
benefits of implicit differentiation and autodiff. To summarize, we make the
following contributions.
\begin{itemize}[topsep=0pt,itemsep=2pt,parsep=2pt,leftmargin=10pt]

\item We describe our framework and its JAX \cite{jax, frostig2018compiling}
    implementation (\url{https://github.com/google/jaxopt/}). 
    Our framework significantly {\bf lowers the
    barrier} to use implicit differentiation, thanks to the
    seamless integration in JAX, with low-level details all abstracted away.

\item We instantiate our framework on a {\bf large catalog} of
    optimality conditions (Table \ref{tab:mapping_summary}), recovering existing
    schemes and obtaining new ones, such as the mirror descent
    fixed point based one.

\item On the theoretical side, we provide new bounds on the {\bfseries Jacobian
    error} when the optimization problem is only solved approximately, and
    empirically validate them.

\item We implement four {\bfseries illustrative applications}, demonstrating our
    framework's ease of use.

\end{itemize}

Beyond our software implementation in JAX, we hope this paper provides a
\textbf{self-contained blueprint} for creating an efficient and modular
implementation of implicit differentiation in other frameworks.

\paragraph{Notation.}

We denote the gradient and Hessian of $f \colon \RR^d \to \RR$ evaluated at $x
\in \RR^d$ by $\nabla f(x) \in \RR^d$ and $\nabla^2 f(x) \in \RR^{d \times d}$.
We denote the Jacobian of $F \colon \RR^d \to \RR^p$ evaluated at $x \in \RR^d$
by $\partial F(x) \in \RR^{p \times d}$. When $f$ or $F$ have several arguments,
we denote the gradient, Hessian and Jacobian in the $i^{\text{th}}$ argument by
$\nabla_i$, $\nabla^2_i$ and $\partial_i$, respectively. The standard
probability simplex is denoted by $\triangle^d \coloneqq \{x \in \RR^d \colon
\|x\|_1 = 1, x \ge 0\}$. 
For any set $\cC\subset\RR^d$, we denote the indicator function $I_\cC \colon \RR^d
\rightarrow \RR\cup\{+\infty\}$ where $I_\cC(x) = 0$ if $x\in\cC$, $I_\cC(x) =
+\infty$ otherwise. For a vector or matrix $A$, we note $\|A\|$ the Frobenius
(or Euclidean) norm, and $\|A\|_{\text{op}}$ the operator norm.

\section{Automatic implicit differentiation}
\label{sec:framework}

\subsection{General principles}
\label{sec:general_principles}

\paragraph{Overview.}

Contrary to autodiff through unrolled algorithm iterations, implicit
differentiation typically involves a manual, sometimes complicated, mathematical
derivation.  For instance, numerous works \cite{chapelle_2002, gould_2016,
amos_2017,sparsemap,lp_sparsemap} use Karush–Kuhn–Tucker (KKT) conditions in
order to relate a constrained optimization problem's solution to its inputs, and
to manually derive a formula for its derivatives.  The derivation and
implementation in these works are typically case-by-case.

In this work, we propose a generic way to easily add implicit
differentiation on top of existing solvers. In our approach, the user defines
directly in Python a mapping function $F$ capturing
the optimality conditions of the problem solved by the algorithm.
We provide reusable building blocks to easily express such $F$.
The provided $F$ is then plugged into our Python decorator
\texttt{@custom\_root}, which we append on top of the solver declaration we wish to
differentiate.  Under the hood, we combine the implicit function theorem and
autodiff of $F$ to automatically differentiate the optimization problem
solution. A simple illustrative example is given in Figure \ref{fig:ridge}.

\begin{figure}[t]
\centering
\fbox{
\includegraphics[scale=0.72]{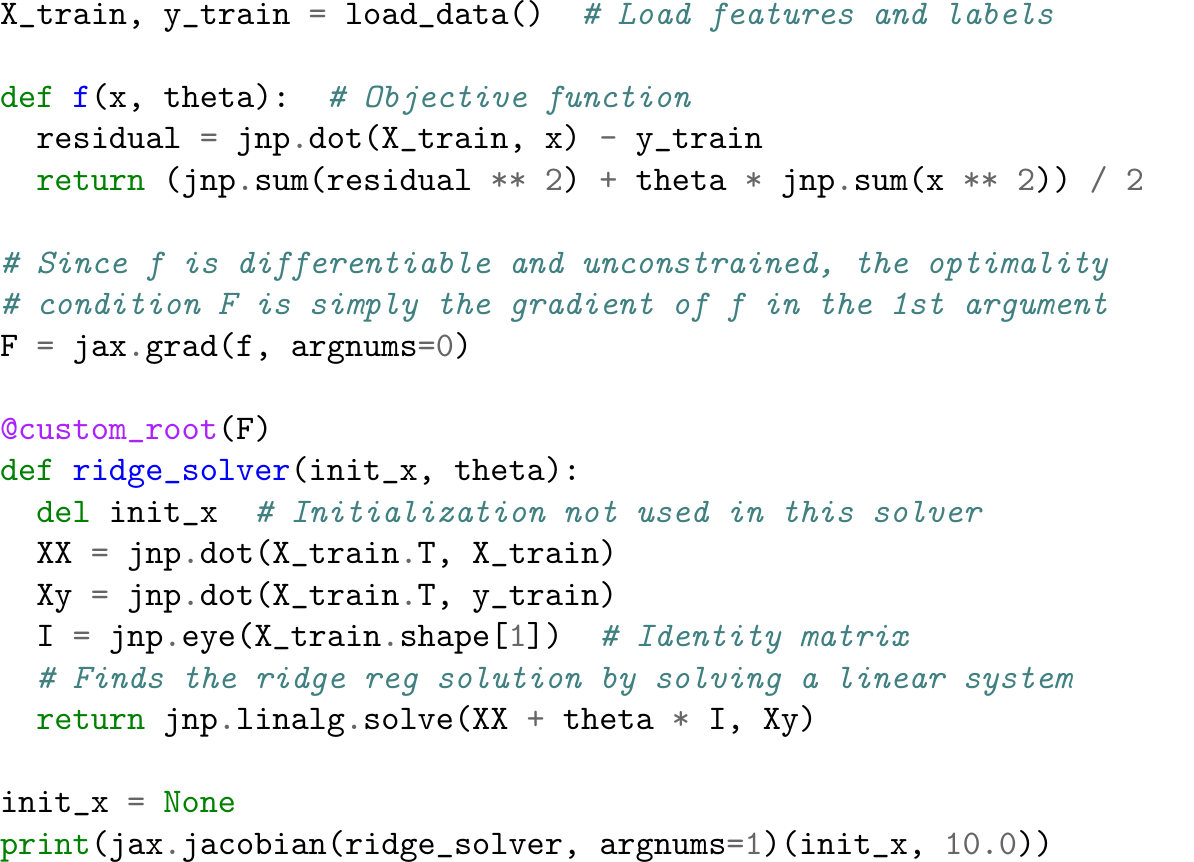}
}
\caption{Adding implicit differentiation on top of a ridge
regression solver. The function $f(x, \theta)$ defines the objective function
and the mapping $F$, here simply equation \eqref{eq:stationary_cond}, captures
the optimality conditions. Our decorator
\texttt{@custom\_root} automatically adds implicit
differentiation to the solver for the user, overriding JAX's default behavior.
The last line evaluates the Jacobian at $\theta = 10$.}
\label{fig:ridge}
\end{figure}

\paragraph{Differentiating a root.}

Let $F \colon \RR^d \times \RR^n \to \RR^d$ be a user-provided mapping,
capturing the optimality conditions of a problem. 
An optimal solution, denoted $x^\star(\theta)$, should be a
\textbf{root} of $F$:
\begin{equation}
F(x^\star(\theta), \theta) = 0\,.
\label{eq:root_pb}
\end{equation}
We can see $x^\star(\theta)$ as an implicitly defined function of $\theta \in
\RR^n$, i.e., $x^\star \colon \RR^n \to \RR^d$.
More precisely, from the \textbf{implicit function theorem}
\cite{griewank_2008,krantz_2012}, we know that for $(x_0, \theta_0)$ satisfying $F(x_0,
\theta_0) = 0$ with a continuously differentiable $F$, if the Jacobian
$\partial_1 F$ evaluated at $(x_0, \theta_0)$ is a square invertible matrix,
then there exists a function $x^\star(\cdot)$ defined on a neighborhood of
$\theta_0$ such that $x^\star(\theta_0) = x_0$. Furthermore, for all $\theta$ in this
neighborhood, we have that $F(x^\star(\theta), \theta) = 0$ and
$\partial x^\star(\theta)$ exists.
Using the chain rule, the Jacobian $\partial x^\star(\theta)$
satisfies
\begin{equation}
\partial_1 F(x^\star(\theta), \theta) \partial x^\star(\theta) + 
\partial_2 F(x^\star(\theta), \theta) = 0\,.
\end{equation}
Computing $\partial x^\star(\theta)$ therefore boils down to the resolution of
the linear system of equations
\begin{equation}
\underbrace{-\partial_1 F(x^\star(\theta), \theta)}_{A \in \RR^{d \times d}} 
\underbrace{\partial x^\star(\theta)}_{J \in \RR^{d \times n}}
= \underbrace{\partial_2 F(x^\star(\theta), \theta)}_{B \in \RR^{d \times n}}.
\label{eq:linear_system_root}
\end{equation}
When \eqref{eq:root_pb} is a one-dimensional root finding problem ($d=1$),
\eqref{eq:linear_system_root} becomes particularly simple since we then have
$\nabla x^\star(\theta) = B^\top / A$, where $A$ is a scalar value.

We will show that existing and new implicit differentiation methods all reduce
to this simple principle.
We call our approach \textbf{automatic implicit differentiation}
as the user can freely express the optimization solution to be
differentiated through the optimality conditions $F$. 
Our approach is \textbf{efficient} as it can be added on top of
any state-of-the-art solver and \textbf{modular} as the optimality condition
specification is \textbf{decoupled} from the implicit differentiation mechanism.
This contrasts with existing works, where the derivation and implementation are
specific to each optimality condition.

\paragraph{Differentiating a fixed point.}

We will encounter numerous applications where $x^\star(\theta)$ is
instead implicitly defined through a \textbf{fixed point}:
\begin{equation}
    x^\star(\theta) = T(x^\star(\theta), \theta)\,,
\end{equation}
where $T \colon \RR^d \times \RR^n \to \RR^d$.
This can be seen as a particular case of \eqref{eq:root_pb} 
by defining the \textbf{residual}
\begin{equation}
F(x, \theta) = T(x, \theta) - x\,.
\label{eq:fixed_point_to_root}
\end{equation}
In this case, when $T$ is continuously differentiable, 
using the chain rule, we have
\begin{equation}
A = -\partial_1 F(x^\star(\theta), \theta) = 
I -\partial_1 T(x^\star(\theta), \theta)
\quad \text{and} \quad
B = \partial_2 F(x^\star(\theta), \theta) = 
\partial_2 T(x^\star(\theta), \theta).
\end{equation}

\paragraph{Computing JVPs and VJPs.}

In most practical scenarios, it is not necessary to explicitly form the Jacobian
matrix, and instead it is sufficient to left-multiply or right-multiply by
$\partial_1 F$ and $\partial_2 F$. These are called vector-Jacobian product
(VJP) and Jacobian-vector product
(JVP), and are useful for integrating $x^\star(\theta)$ with reverse-mode and
forward-mode autodiff, respectively.
Oftentimes, $F$ will be explicitly defined. In this case, computing the VJP or
JVP can be done via autodiff. 
In some cases, $F$ may itself be implicitly defined, for instance
when $F$ involves the solution of a variational problem. In this case, computing
the VJP or JVP will itself involve implicit differentiation.

The right-multiplication (JVP) between $J = \partial x^\star(\theta)$ and a
vector $v$, $Jv$, can be computed efficiently by solving $A (Jv) = Bv$. 
The left-multiplication (VJP) of $v^\top$ with $J$,
$v^\top J$, can be computed by first solving $A^\top u = v$. Then, we can obtain
$v^\top J$ by $v^\top J = u^\top A J = u^\top B$. 
Note that when $B$ changes but $A$ and $v$ remain the same, we
do not need to solve $A^\top u = v$ once again. This allows to compute the VJP
w.r.t. different variables while solving only one linear system.

To solve these linear systems, we can use the conjugate gradient method
\cite{conjugate_gradient} when $A$ is symmetric positive semi-definite and GMRES
\cite{saad_1986} or BiCGSTAB \cite{Vorst1992-bicgstab} otherwise. These
algorithms are all matrix-free: they only require matrix-vector products.  Thus,
all we need from $F$ is its JVPs or VJPs.  An alternative to GMRES/BiCGSTAB is
to solve the normal equation $A A^\top u = A v$ using conjugate gradient.  This
can be
implemented using JAX's transpose routine
\texttt{jax.linear\_transpose} \cite{frostig2021decomposing}.
In case of non-invertibility, a common heuristic is to solve a least
squares $\min_J \|AJ-B\|^2$ instead.

\paragraph{Pre-processing and post-processing mappings.}

Oftentimes, the goal is not to differentiate $\theta$ per se, but the
parameters of a function producing $\theta$. One example of such pre-processing
is to convert the parameters to be differentiated from one form to another
canonical form, such as a quadratic program \cite{amos_2017} or a conic program
\cite{agrawal_2019}. 
Another example
is when $x^\star(\theta)$ is used as the output of a neural network layer, in
which case $\theta$ is produced by the previous layer. Likewise,
$x^\star(\theta)$ will often not be the final output we want to differentiate.
One example of such post-processing is when $x^\star(\theta)$ is the solution of
a dual program and we apply the dual-primal mapping to recover the solution of
the primal program.  Another example is the application of a loss function, in
order to reduce $x^\star(\theta)$ to a scalar value. We leave the
differentiation of such pre/post-processing mappings to the autodiff system,
allowing to compose functions in complex ways.

\paragraph{Implementation details.}

%
%

When a solver function is decorated with \texttt{@custom\_root}, we use
\texttt{jax.custom\_jvp} and \texttt{jax.custom\_vjp} to automatically add
custom JVP and VJP rules to the function, overriding JAX's default behavior.
As mentioned above, we use linear system solvers based on matrix-vector products
and therefore we only need access to $F$ through the JVP or VJP
with $\partial_1 F$ and $\partial_2 F$. This is done by using \texttt{jax.jvp} and
\texttt{jax.vjp}, respectively. Note that, as in Figure \ref{fig:ridge}, the
definition of $F$ will often include a gradient mapping $\nabla_1 f(x, \theta)$.
Thankfully, JAX supports second-order derivatives
transparently.  For convenience, our library also provides a
\texttt{@custom\_fixed\_point} decorator, for adding implicit differentiation on
top of a solver, given a fixed point iteration $T$; see code examples in
Appendix \ref{appendix:code_examples}.  

\begin{table}[t]
\caption{Summary of optimality condition mappings. Oracles are accessed through
their JVP or VJP.
}
\begin{center}
\begin{small}
\begin{tabular}{cccc}
\toprule
Name & Equation & Solution needed & Oracle \\
\midrule
Stationary & \eqref{eq:stationary_cond}, \eqref{eq:gradient_descent_fp} & 
Primal & $\nabla_1 f$ \\
KKT & \eqref{eq:kkt_conditions} & Primal \textit{and} dual & 
$\nabla_1 f$, $H$, $G$, $\partial_1 H$, $\partial_1 G$ \\
Proximal gradient & \eqref{eq:proximal_grad_fp} & Primal & 
$\nabla_1 f$, $\prox_{\eta g}$ \\
Projected gradient & \eqref{eq:proj_grad_fp} & Primal & $\nabla_1 f$, 
$\proj_{\cC}$ \\
Mirror descent & \eqref{eq:mirror_descent_fp} & Primal & $\nabla_1 f$,
$\proj_\cC^\varphi$, $\nabla \varphi$ \\
Newton & \eqref{eq:newton_opt_fp} & Primal &
$[\nabla^2_1 f(x, \theta)]^{-1}$, $\nabla_1 f(x, \theta)$ \\
Block proximal gradient & \eqref{eq:bcd_fp} & Primal &
$[\nabla_1 f]_j$, $[\prox_{\eta g}]_j$ \\
Conic programming & \eqref{eq:residual_map} & Residual map root & 
$\proj_{\RR^p \times \cK^* \times \RR_+}$ \\
\bottomrule
\end{tabular}
\end{small}
\end{center}
\label{tab:mapping_summary}
\end{table}

\subsection{Examples}
\label{sec:mapping_examples}

We now give various examples of mapping $F$ or fixed point iteration $T$,
recovering existing implicit differentiation methods and creating new ones.
Each choice of $F$ or $T$ implies different trade-offs in terms
of \textbf{computational oracles};  see Table
\ref{tab:mapping_summary}. 
Source code examples are given in Appendix \ref{appendix:code_examples}.

\paragraph{Stationary point condition.}

The simplest example is to differentiate through the implicit function
\begin{equation}
x^\star(\theta) = \argmin_{x \in \RR^d} f(x, \theta),
\end{equation}
where $f \colon \RR^d \times \RR^n \to \RR$ is twice differentiable, $\nabla_1
f$ is continuously differentiable, and $\nabla^2_1 f$ is invertible at
$(x^\star(\theta),\theta)$.
In this case, $F$ is simply the gradient mapping
\begin{equation}
F(x, \theta) = \nabla_1 f(x, \theta).
\label{eq:stationary_cond}
\end{equation}
We then have
$\partial_1 F(x, \theta) = \nabla^2_1 f(x, \theta)$
and
$\partial_2 F(x, \theta) = \partial_2 \nabla_1 f(x, \theta)$,
the Hessian of $f$ in its first argument and the
Jacobian in the second argument of $\nabla_1 f(x, \theta)$. In practice,
we use autodiff to compute Jacobian products automatically.
Equivalently, we can use the \textbf{gradient descent fixed point}
\begin{equation}
T(x, \theta) = x - \eta \nabla_1 f(x, \theta),
\label{eq:gradient_descent_fp}
\end{equation}
for all $\eta > 0$. Using \eqref{eq:fixed_point_to_root}, 
it is easy to check that we obtain
the same linear system since $\eta$ cancels out.

\paragraph{KKT conditions.}

As a more advanced example,
we now show that the KKT conditions, manually differentiated in several works 
\cite{chapelle_2002,gould_2016,amos_2017,sparsemap,lp_sparsemap}, fit
our framework. As we will see, the key will be to group the
optimal primal and dual variables as our $x^\star(\theta)$.
Let us consider the general problem
\begin{equation}
\argmin_{z \in \RR^p} f(z, \theta)
\quad \text{subject to} \quad
G(z, \theta) \le 0,
~ H(z, \theta) = 0,
\label{eq:generic_constrained_pb}
\end{equation}
where $z \in \RR^p$ is the primal variable,
$f \colon \RR^p \times \RR^n \to \RR$,
$G \colon \RR^p \times \RR^n \to \RR^r$
and $H \colon \RR^p \times \RR^n \to \RR^q$
are twice differentiable convex functions, and $\nabla_1 f$, $\partial_1 G$ and $\partial_1 H$ are continuously differentiable.
The stationarity, primal feasibility and complementary slackness conditions give
\begin{align}
\nabla_1 f(z, \theta) + [\partial_1 G(z, \theta)]^\top \lambda + 
[\partial_1 H(z, \theta)]^\top \nu = 0 \\
H(z, \theta) = 0 \\
\lambda \circ G(z, \theta) = 0,
\label{eq:kkt_conditions}
\end{align}
where $\nu \in \RR^q$ and $\lambda \in \RR^r_+$ are the dual variables, also
known as KKT multipliers.
The primal and dual feasibility conditions can be ignored almost everywhere
\cite{tutorial_implicit}.
The system of (potentially nonlinear) equations \eqref{eq:kkt_conditions}
fits our framework, as we can group the primal and dual solutions as
$x^\star(\theta) = (z^\star(\theta), \nu^\star(\theta), \lambda^\star(\theta))$
to form the root of a function $F(x^\star(\theta), \theta)$, where $F \colon
\RR^d \times \RR^n \to \RR^d$ and $d = p + q + r$. The primal and dual solutions
can be obtained from a generic solver, such as an interior point method.
In practice, the above mapping $F$ will be defined directly in Python
(see Figure \ref{fig:kkt_code} in Appendix \ref{appendix:code_examples}) and
$F$ will be differentiated automatically via autodiff.

\paragraph{Proximal gradient fixed point.}

Unfortunately, not all algorithms return both primal and dual solutions. 
Moreover, if the objective contains non-smooth terms,
proximal gradient descent may be more efficient.
We now discuss its fixed point \cite{niculae_2017, bertrand_2020_implicit,
bertrand_2021_journal}.
Let $x^\star(\theta)$ be implicitly defined as
\begin{equation}
x^\star(\theta) \coloneqq \argmin_{x \in \RR^d} f(x, \theta) + g(x, \theta),
\label{eq:composite_pb}
\end{equation}
where $f \colon \RR^d \times \RR^n \to \RR$ is twice-differentiable convex
and $g \colon \RR^d \times \RR^n \to \RR$ is convex but possibly non-smooth.
Let us define the proximity operator associated with $g$ by
\begin{equation}
\prox_g(y, \theta) \coloneqq 
\argmin_{x \in \RR^d} \frac{1}{2} \|x - y\|^2_2 + g(x, \theta).
\end{equation}
To implicitly differentiate $x^\star(\theta)$, we use the fixed point
mapping \cite[p.150]{parikh_2014}
\begin{equation}
T(x, \theta) = \prox_{\eta g}(x - \eta \nabla_1 f(x, \theta), \theta),
\label{eq:proximal_grad_fp}
\end{equation}
for any step size $\eta > 0$.
The proximity operator is
$1$-Lipschitz continuous \cite{Moreau1965Proximite}. By Rademacher's theorem, it
is differentiable almost
everywhere. If, in addition, it is continuously differentiable in a
neighborhood of $(x^\star(\theta),\theta)$ and if $I - \partial_1
T(x^\star(\theta),\theta)$ is invertible, then our framework to differentiate
$x^\star(\theta)$ applies. Similar assumptions are made in
\cite{bertrand_2021_journal}.
Many proximity operators enjoy a closed form and can
easily be
differentiated, as discussed in Appendix \ref{appendix:jac_prod}.
An implementation is given in Figure \ref{fig:pg_fixed_point}.

\paragraph{Projected gradient fixed point.} 

As a special case,
when $g(x, \theta)$ is the indicator function $I_{\cC(\theta)}(x)$,
where $\cC(\theta)$ is a convex set depending on $\theta$, we obtain
\begin{equation}
x^\star(\theta) = \argmin_{x \in \cC(\theta)} f(x, \theta).
\label{eq:constrained_pb}
\end{equation}
The proximity operator $\prox_g$ becomes the Euclidean projection onto
$\cC(\theta)$
\begin{equation}
\prox_g(y, \theta) =
\proj_{\cC}(y, \theta) \coloneqq
\argmin_{x \in \cC(\theta)} \|x - y \|^2_2
\label{eq:projection}
\end{equation}
and \eqref{eq:proximal_grad_fp} becomes the projected gradient fixed point
\begin{equation}
T(x, \theta) = \proj_{\cC}(x - \eta \nabla_1 f(x, \theta), \theta).
\label{eq:proj_grad_fp}
\end{equation}
Compared to the KKT conditions,
this fixed point is particularly suitable when the
projection enjoys a closed form.
We discuss how to compute the JVP / VJP for a wealth of convex sets in Appendix
\ref{appendix:jac_prod}. 

\begin{figure}[t]
\centering
\fbox{
\includegraphics[scale=0.8]{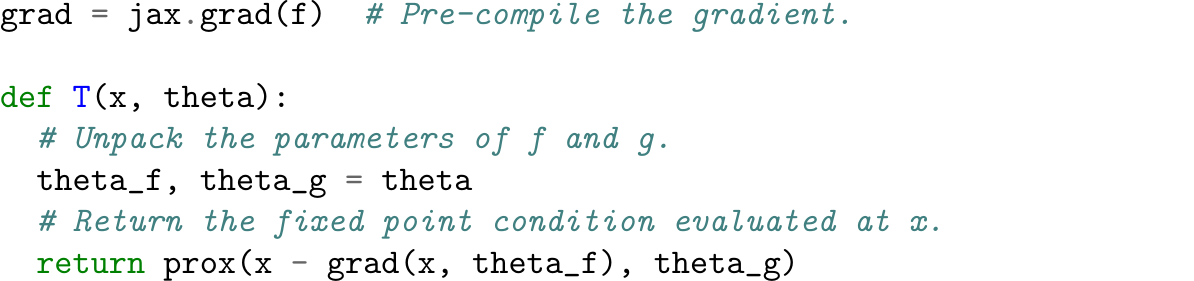}
}
\caption{Implementation of the proximal gradient fixed point 
\eqref{eq:proximal_grad_fp} with step size $\eta=1$.}
\label{fig:pg_fixed_point}
\end{figure}

\paragraph{Current limitations.}

While we have not observed issues in practice, we note that the approach
developed in this section theoretically only applies to settings where
the implicit function theorem is valid, namely, where optimality conditions
satisfy the differentiability and invertibility conditions stated in
\S\ref{sec:general_principles}. While this covers a wide
range of situations even for non-smooth optimization problems (e.g., under mild assumptions the
solution of a Lasso regression can be differentiated a.e. with respect to the
regularization parameter, see Appendix~\ref{sec:lasso}), an interesting direction for future work is to extend
the framework to handle cases where the differentiability and invertibility
conditions are not satisfied, using, e.g., the theory of nonsmooth implicit
function theorems \cite{Clarke1983Optimization,Bolte2021Nonsmooth}.

\section{Jacobian precision guarantees}
\label{sec:jac_bounds}


In practice, either by the limitations of finite precision arithmetic or because
we perform a finite number of iterations, we rarely reach the exact solution
$x^\star(\theta)$. Instead, we reach an approximate solution
$\hat{x}$ and apply the implicit differentiation equation 
\eqref{eq:linear_system_root} at this approximate solution. This motivates the
need for precision guarantees of this approach. We introduce the following
formalism.

\begin{definition}
\label{DEF:jac-est}
Let $F:\RR^d \times \RR^n \to \RR^d$ be a continously differentiable optimality criterion mapping.
Let $A \coloneqq -\partial_1 F$ and $B \coloneqq \partial_2 F$.
We define the \textbf{Jacobian estimate} at $(x, \theta)$, when $A(x,\theta)$ is
invertible,
as the solution to the linear equation
$A(x, \theta) J(x, \theta) = B(x, \theta)$.
It is a function $J: \RR^d \times \RR^n \to \RR^{d \times n}$.
\end{definition}
It holds by construction that
$J(x^\star(\theta), \theta) = \partial x^\star(\theta)$.
Computing $J(\hat x, \theta)$ for an approximate solution $\hat
x$ of $x^\star(\theta)$ therefore
allows to approximate the true Jacobian
$\partial x^\star(\theta)$. In practice, an algorithm used to solve
\eqref{eq:root_pb} depends on $\theta$. Note however that, what we
compute is not the Jacobian of $\hat x(\theta)$, unlike works
differentiating through unrolled algorithm iterations, but an
estimate of $\partial x^\star(\theta)$. We therefore use the notation
$\hat x$, leaving the dependence on $\theta$ implicit.

We develop bounds of the form $\|J(\hat{x}, \theta) - \partial x^\star(\theta)\|
< C \|\hat x - x^\star(\theta)\|$, hence showing that the error on the estimated
Jacobian is at most of the same order as that of $\hat x$ as an approximation of
$x^{\star}(\theta)$. These bounds are based on the following main theorem, whose
proof is included in Appendix \ref{appendix:proofs}.
\begin{theorem}\label{thm:jacob}
Let $F:\RR^d \times \RR^n \to \RR^d$ be continuously differentiable. If there are
$\alpha, \beta, \gamma, \varepsilon, R>0$ s.t. $A = -\partial_1 F$ and $B =
\partial_2 F$ satisfy, for all $v\in\RR^d$, $\theta \in \RR^n$ and $x$ s.t. 
$\|x - x^\star(\theta)\| \le \varepsilon$:

$A$ is well-conditioned, Lipschitz: $\|A(x, \theta) v \| \ge \alpha \|v\|$ ,
$\|A(x, \theta) - A(x^\star(\theta), \theta)\|_{\textnormal{op}} \le \gamma \|x -
x^\star(\theta)\|$.

$B$ is bounded and Lipschitz: $\|B(x^\star(\theta), \theta)\| \le R$ , $\|B(x, \theta) - B(x^\star(\theta), \theta)\| \le \beta \|x - x^\star(\theta)\|$.

Under these conditions, when $\|\hat x - x^\star(\theta)\| \le \varepsilon$, we have
\[
\|J(\hat x, \theta) - \partial x^\star (\theta)\| \le \left(\beta\alpha^{-1} + \gamma R \alpha^{-2}\right) \|\hat x - x^\star(\theta)\|\, .
\]
\end{theorem}

This result is inspired by \cite[Theorem 7.2]{Higham2002Accuracy}, that is
concerned with the stability of solutions to inverse problems. As a difference, 
we consider
that $A(\cdot, \theta)$ is uniformly well-conditioned, rather than only at
$x^\star(\theta)$. This does not affect the first order in $\varepsilon$ of this
bound, and makes it valid for all $\hat x$. 
Our goal with Theorem~\ref{thm:jacob} is to provide a result
that works for general $F$ but can be tailored to specific cases. 

In particular, for the gradient descent fixed point
\eqref{eq:gradient_descent_fp},
this yields
\[
A(x, \theta) = \eta \nabla_1^2 f(x, \theta)\,  \text{and} \; B(x, \theta) = 
- \eta \partial_2 \nabla_1 f(x, \theta)\, .
\]
By specializing Theorem~\ref{thm:jacob} for this fixed point, 
we obtain Jacobian precision guarantees with conditions directly on $f$ rather
than $F$; see Corollary~\ref{cor:precision-gd} in Appendix
\ref{appendix:proofs}.
These guarantees hold for instance for the dataset distillation experiment
in Section \ref{sec:exp}.
Our analysis reveals in particular that Jacobian estimation by implicit differentiation
\textbf{gains a factor of} $\mathbf{t}$ \textbf{compared to automatic
differentiation}, after $t$ iterations of gradient descent in the
strongly-convex setting \cite[Proposition 3.2]{ablin_2020}. 
While our guarantees concern the Jacobian of $x^\star(\theta)$, we note that
other studies \cite{grazzi_2020,ji_2021,bertrand_2021_journal} give
guarantees on hypergradients (i.e., the gradient of an outer objective).

\begin{wrapfigure}[13]{r}{0.41\textwidth}
    \centering
    \vspace{-0.5cm}
    \includegraphics[width=0.4\textwidth]{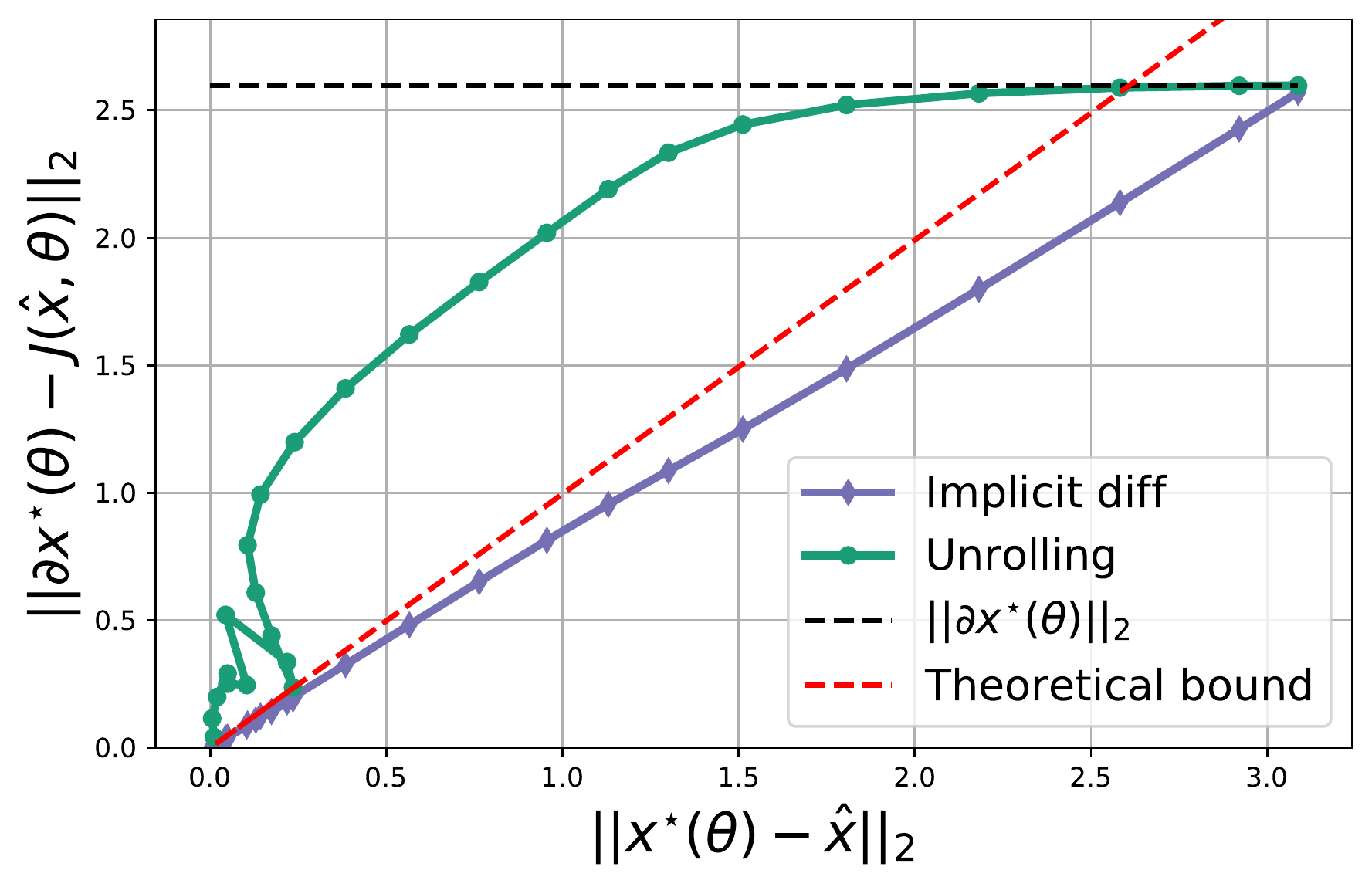}\\
    \caption{Jacobian estimate errors. Empirical error of implicit
    differentiation follows closely the theoretical upper bound. 
Unrolling achieves a much worse error for comparable iterate error.}
    \label{fig:jac-precision}
\end{wrapfigure}

We illustrate these results on ridge regression, where
$x^\star(\theta) = \argmin_x \|\Phi x -y\|^2_2 + \sum_i
\theta_i x_i^2$. 
This problem has the merit that the solution $x^\star(\theta)$ and its Jacobian
$\partial x^\star(\theta)$ are available in closed form.
By running gradient descent for $t$ iterations,
we obtain an estimate $\hat x$ of $x^\star(\theta)$ and an estimate $J(\hat x,
\theta)$ of $\partial x^\star(\theta)$; cf. Definition \ref{DEF:jac-est}.
By doing so for different numbers of iterations $t$, we can graph the relation
between the error $\|x^\star(\theta) - \hat x\|_2$ and the error
$\|\partial x^\star(\theta) - J(\hat x, \theta)\|_2$, as shown in Figure
\ref{fig:jac-precision},
empirically validating Theorem~\ref{thm:jacob}.
The results in Figure \ref{fig:jac-precision} were obtained using the diabetes
dataset from~\cite{efron2004least}, with other datasets yielding a qualitatively
similar behavior. We derive similar guarantees in
Corollary~\ref{cor:precision-prox} in Appendix \ref{appendix:proofs} for
proximal gradient descent.

\section{Experiments}
\label{sec:exp}

In this section, we demonstrate the ease of solving bi-level
optimization problems with our framework. We also present an
application to the sensitivity analysis of molecular dynamics.

\subsection{Hyperparameter optimization of multiclass SVMs}

In this example, we consider the hyperparameter optimization of multiclass SVMs
\cite{crammer_2001}
trained in the dual. Here, $x^\star(\theta)$ is the optimal dual
solution, a matrix of shape $m \times k$, where $m$ is the number of training
examples and $k$ is the number of classes, and $\theta \in \RR_+$ is the
regularization parameter. The challenge
in differentiating $x^\star(\theta)$ is that each row of $x^\star(\theta)$
is constrained to belong to the probability simplex $\triangle^k$. More formally,
let $X_\training \in \RR^{m \times p}$ be the training feature matrix and
$Y_\training \in 
\{0, 1\}^{m \times k}$ be the training labels (in row-wise one-hot encoding). 
Let $W(x, \theta) \coloneqq X_\training^\top (Y_\training - x) / \theta \in
\RR^{p \times k}$ be the dual-primal mapping.
Then, we consider the following bi-level optimization problem
\begin{equation}
\small
\underbrace{\min_{\theta = \exp(\lambda)}
\frac{1}{2} \|X_\validation W(x^\star(\theta), \theta) -
Y_\validation\|^2_F}_{\text{outer problem}}
\quad \text{subject to} \quad
\underbrace{x^\star(\theta) = \argmin_{x \in \cC} 
f(x, \theta) \coloneqq \frac{\theta}{2} \|W(x, \theta)\|^2_F
+ \langle x, Y_\training \rangle}_{\text{inner problem}},
\label{eq:multiclass_svm_obj}
\end{equation}
where $\cC = \triangle^k \times \dots \times \triangle^k$ is the Cartesian
product of $m$ probability simplices.
We apply the change of variable $\theta = \exp(\lambda)$ in order to guarantee
that the hyper-parameter $\theta$ is positive.
The matrix $W(x^\star(\theta), \theta) \in \RR^{p
\times k}$ contains the optimal primal solution, the feature weights for each
class.  The outer loss is computed against validation data $X_\validation$ and
$Y_\validation$. 

\begin{figure}[t]
\begin{subfigure}{.33\textwidth}
  \centering
  \includegraphics[width=\linewidth]{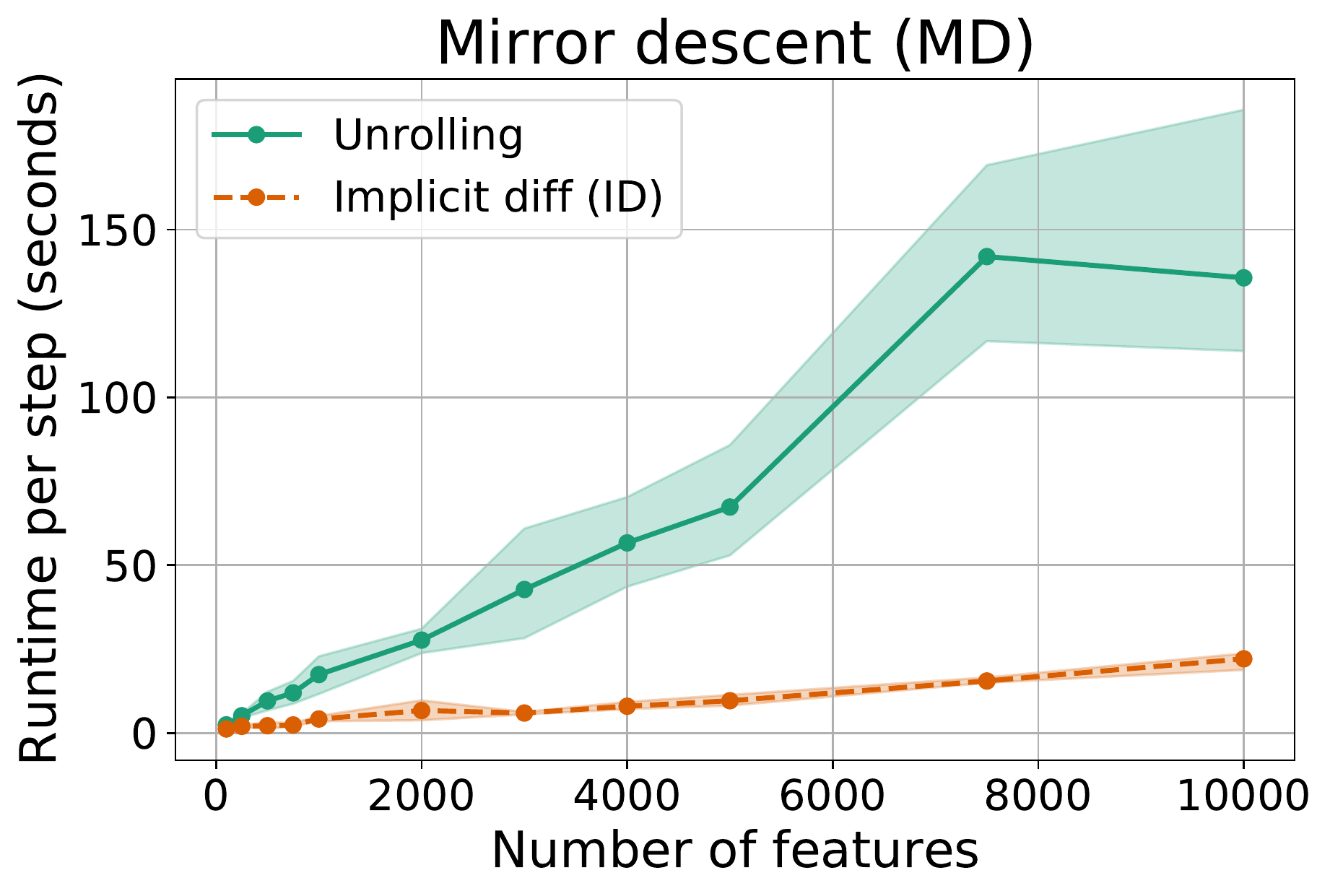}
  \caption{}
  \label{fig:multiclass_svm_runtime_cpu_a}
\end{subfigure}
\begin{subfigure}{.33\textwidth}
  \centering
  \includegraphics[width=\linewidth]{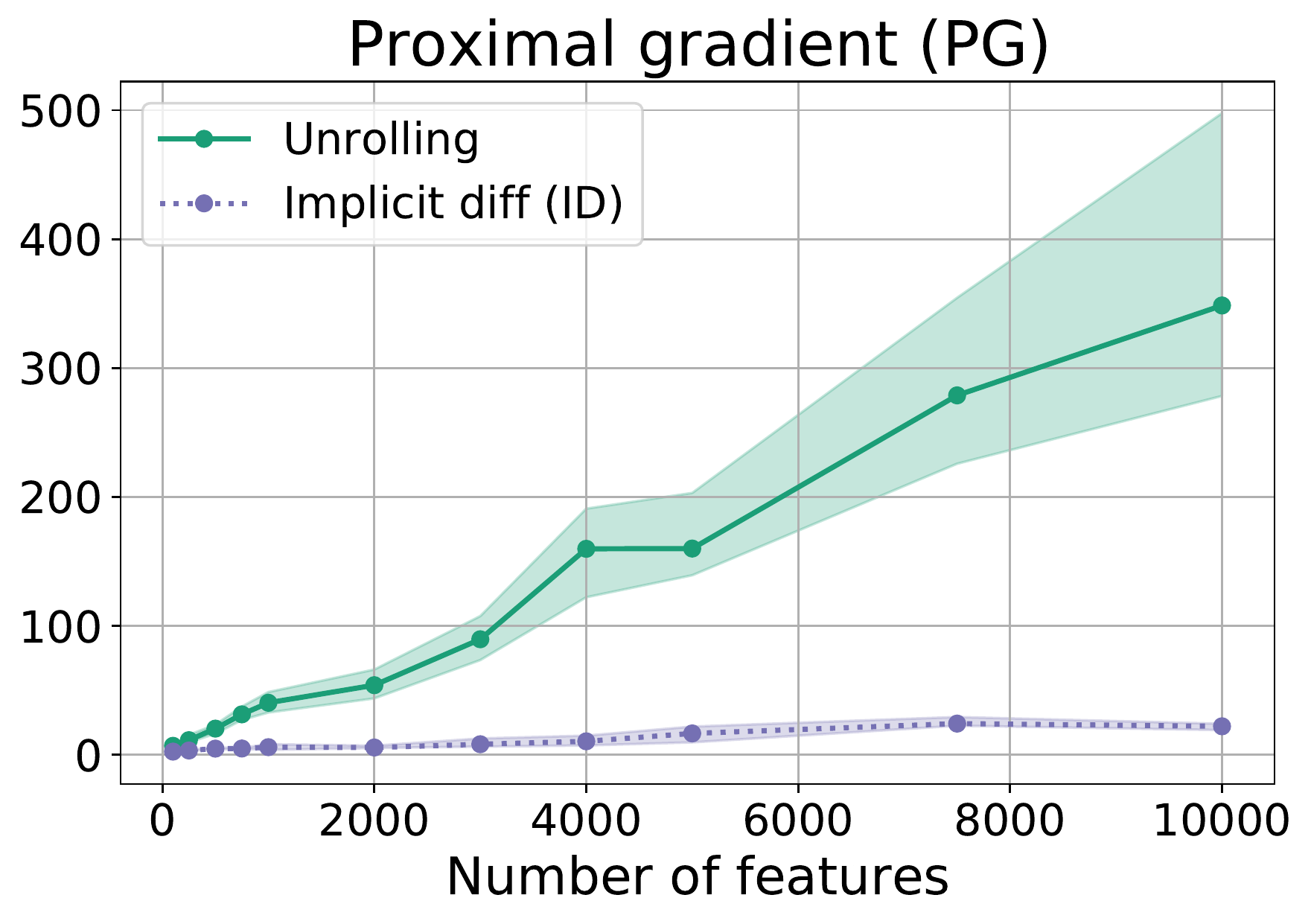}
  \caption{}
  \label{fig:multiclass_svm_runtime_cpu_b}
\end{subfigure}
\begin{subfigure}{.33\textwidth}
  \centering
  \includegraphics[width=\linewidth]{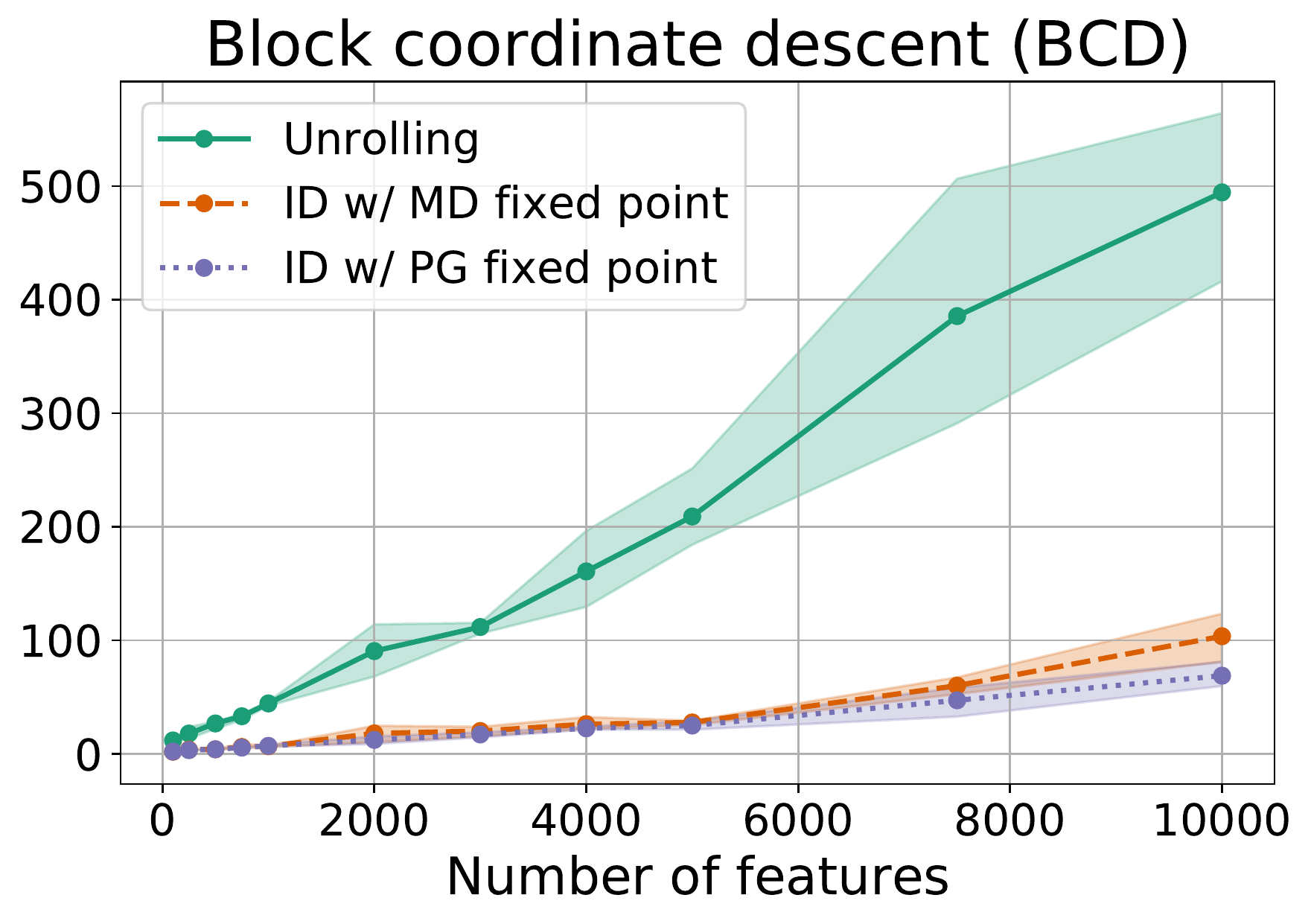}
  \caption{}
  \label{fig:multiclass_svm_runtime_cpu_c}
\end{subfigure}
\caption{CPU runtime comparison of implicit differentiation and unrolling for
hyperparameter optimization of multiclass SVMs for multiple problem sizes. Error
bars represent 90\% confidence
intervals. {\bf (a)} Mirror descent (MD) solver, with MD fixed point for
differentiation. {\bf (b)} Proximal gradient (PG) solver, with PG
fixed point for differentiation. {\bf (c)} Block coordinate
descent solver;
for implicit differentiation we obtain $x^\star(\theta)$ by BCD
but perform differentiation with the MD 
and PG fixed points. This
shows that the solver and fixed point can be independently chosen. }
\label{fig:multiclass_svm_runtime_cpu}
\end{figure}

While KKT conditions can be used to differentiate $x^\star(\theta)$, a more
direct way is to use the projected gradient fixed point \eqref{eq:proj_grad_fp}.
The projection onto $\cC$ can be easily computed by row-wise projections on the
simplex. The projection's Jacobian enjoys a closed form (Appendix
\ref{appendix:jac_prod}).  Another way to differentiate $x^\star(\theta)$
is using the mirror descent fixed point \eqref{eq:mirror_descent_fp}.  Under the
KL geometry, projections correspond to a row-wise softmax.
They are therefore easy to compute and differentiate.  Figure
\ref{fig:multiclass_svm_runtime_cpu} compares the runtime performance of
implicit differentiation vs.\ unrolling for the latter two fixed points.

\subsection{Dataset distillation}
\label{sec:distillation}

Dataset distillation \cite{wang2018dataset,lorraine_2020} aims to
learn a small synthetic training dataset such that a model trained on this
learned data set achieves a small loss on the original training set. 
Formally, let
$X_\training \in \RR^{m \times p}$ and $y_\training \in [k]^m$ 
denote the original training set. 
The distilled dataset will contain one prototype example
for each class and therefore $\theta \in \RR^{k \times p}$. 
The dataset distillation problem can then naturally be cast as a
bi-level problem, where in the inner problem we estimate a logistic regression
model $x^\star(\theta) \in \RR^{p \times k}$ trained on the distilled images $\theta \in
\RR^{k \times p}$, while in the outer problem we want to minimize the loss
achieved by $x^\star(\theta)$ over the training set: 
\begin{equation}
\underbrace{\min_{\theta \in \mathbb{R}^{k \times p}} f(x^\star(\theta),
X_\training; y_\training)}_{\text{outer problem}} ~\text{ subject to }~
x^\star(\theta) \in \underbrace{\argmin_{x \in \RR^{p \times k}} f(x, \theta; [k]) +
\varepsilon \|x\|^2\,}_{\text{inner problem}},
\label{eq:bilevel_distillation}
\end{equation}
where $f(W, X; y) \coloneqq \ell(y, XW)$, $\ell$ denotes the
multiclass logistic regression loss, and
$\varepsilon = 10^{-3}$ is a regularization parameter that we found had
a very positive effect on convergence.

\begin{wrapfigure}[16]{r}{0.36\textwidth}
    \centering
    \vspace{-0.5cm}
    \includegraphics[width=0.35\textwidth]{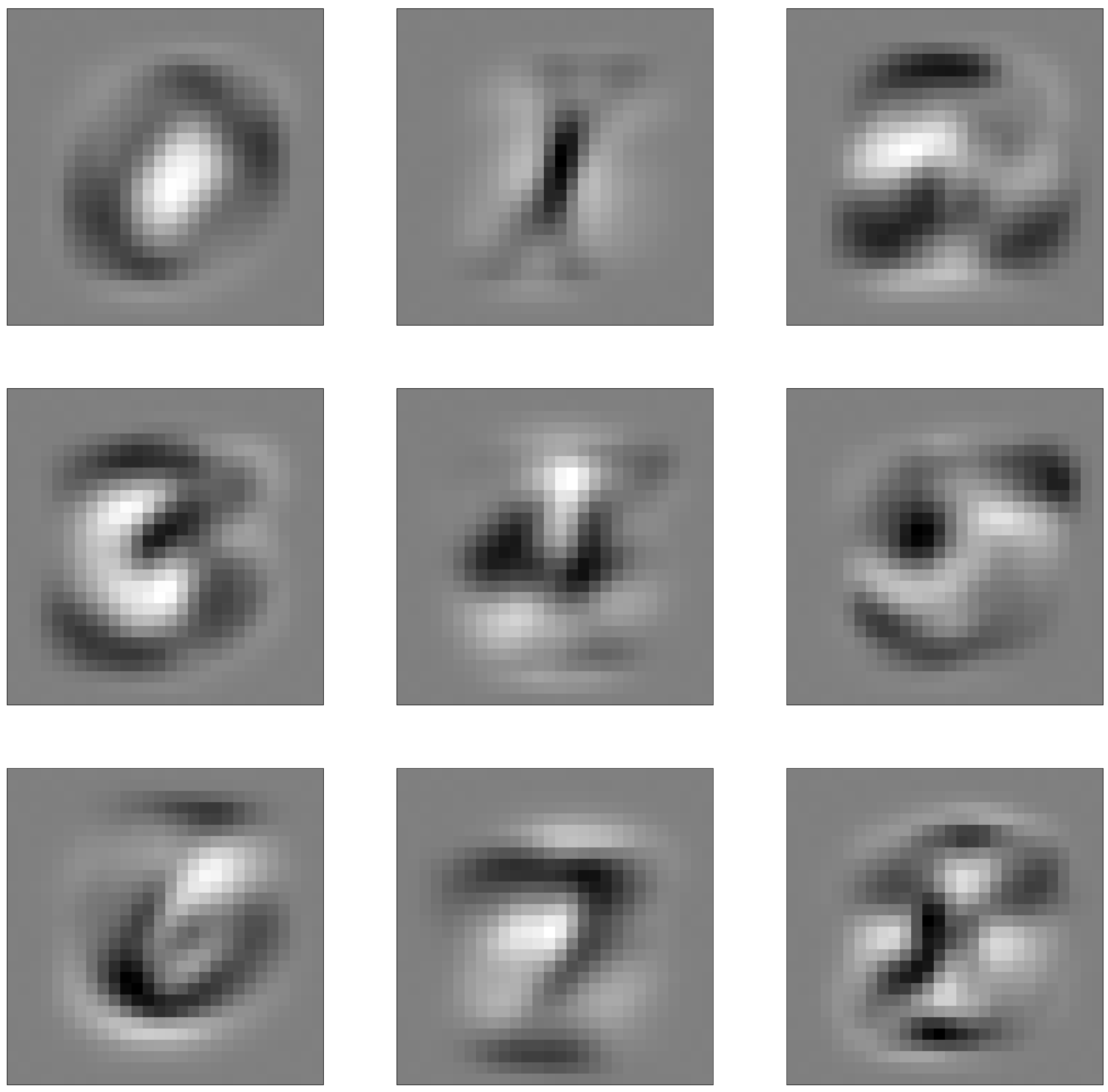}\\
\caption{Distilled dataset $\theta \in \RR^{k \times p}$ obtained by
    solving \eqref{eq:bilevel_distillation}. 
}
    \label{fig:distillation}
\end{wrapfigure}

In this problem, and unlike in the general hyperparameter optimization setup,
\emph{both} the inner and outer problems are high-dimensional, making it an
ideal test-bed for gradient-based bi-level optimization methods. 
For this experiment, we use the MNIST dataset.
The number of
parameters in the inner problem is $p = 28^2 = 784$.  while the number of
parameters of the outer loss is $k\times p = 7840$.
We solve this problem using gradient descent on both the inner and outer
problem, with the gradient of the outer loss computed using implicit
differentiation, as described in \S\ref{sec:framework}. This is
fundamentally different from the approach used in the original paper, where they
used differentiation of the unrolled iterates instead.
For the same solver, we found that the implicit differentiation approach was 4
times faster than the original one.  The obtained distilled images $\theta$ are
visualized in Figure \ref{fig:distillation}.

\subsection{Task-driven dictionary learning}
\label{sec:sparse-coding}

Task-driven dictionary learning was proposed to learn sparse codes for input
data in such a way that the codes solve an outer learning
problem~\cite{mairal_2012,sprechmann2014supervised,zarka2019deep}.
Formally, given a data matrix $X_\training \in \RR^{m \times p}$ 
and a dictionary of $k$ atoms $\theta \in \RR^{k \times p}$, a sparse code is
defined as a matrix $x^\star(\theta)\in\mathbb{R}^{m\times k}$ that minimizes in
$x$ a reconstruction loss $f(x, \theta) \coloneqq \ell(X_\training, x\theta)$
regularized by a sparsity-inducing penalty $g(x)$. Instead of optimizing the
dictionary $\theta$ to minimize the reconstruction loss, \cite{mairal_2012}
proposed to optimize an outer problem that depends on the code. 
Given a set of labels $Y_\training \in \{0, 1\}^{m}$, we consider a logistic
regression problem which results in the bilevel optimization problem:
\begin{equation}\label{eq:bilevel_sparse_coding}
\underbrace{\min_{\theta \in \mathbb{R}^{k \times p}, w\in\mathbb{R}^k, b\in\RR}
\sigma(x^\star(\theta)w + b; y_\training)}_{\text{outer problem}} ~\text{
subject to }~
x^\star(\theta) \in \underbrace{\argmin_{x \in \mathbb{R}^{m\times k}} f(x,
\theta) + g(x)}_{\text{inner problem}}\,.
\end{equation}
When $\ell$ is the squared Frobenius distance between matrices, and $g$ the
elastic net penalty,~\cite[Eq. 21]{mairal_2012} derive manually, using
optimality conditions (notably the support of the codes selected at the
optimum), an explicit re-parameterization of $x^\star(\theta)$ as a linear
system involving $\theta$. This closed-form allows for a \textit{direct}
computation of the Jacobian of $x^\star$ w.r.t. $\theta$. Similarly,
\cite{sprechmann2014supervised} derive first order conditions in the case where
$\ell$ is a $\beta$-divergence, while \cite{zarka2019deep} propose to use
unrolling of ISTA iterations. Our approach bypasses all of these manual
derivations, giving the user more leisure to focus directly on modeling (loss,
regularizer) aspects.

We illustrate this on breast cancer survival prediction from gene
expression data. We frame it as a binary classification problem to discriminate
patients who survive longer than 5 years ($m_1=200$) vs patients who die within
5 years of diagnosis ($m_0=99$), from $p=1,000$ gene expression values. As shown
in Table~\ref{tab:cancerAUC}, solving \eqref{eq:bilevel_sparse_coding}
(Task-driven DictL) reaches a classification performance competitive with
state-of-the-art $L_1$ or $L_2$ regularized logistic regression with 100 times
fewer variables.  

\begin{table}[t]
\caption{Mean AUC (and 95\% confidence interval) for the cancer survival prediction problem.}
\begin{center}
\begin{tabular}{c|cccc}
\toprule
Method & $L_1$ logreg & $L_2$ logreg & DictL + $L_2$ logreg & Task-driven DictL \\
\midrule
AUC (\%)& $71.6 \pm 2.0$ & $72.4 \pm 2.8$ & $68.3 \pm 2.3$ & $73.2 \pm 2.1$\\
\bottomrule
\end{tabular}
\end{center}
\label{tab:cancerAUC}
\end{table}

\subsection{Sensitivity analysis of molecular dynamics}

\begin{wrapfigure}[16]{r}{0.31\textwidth}
    \centering
    \vspace{-0.8cm}
    \includegraphics[width=0.30\textwidth]{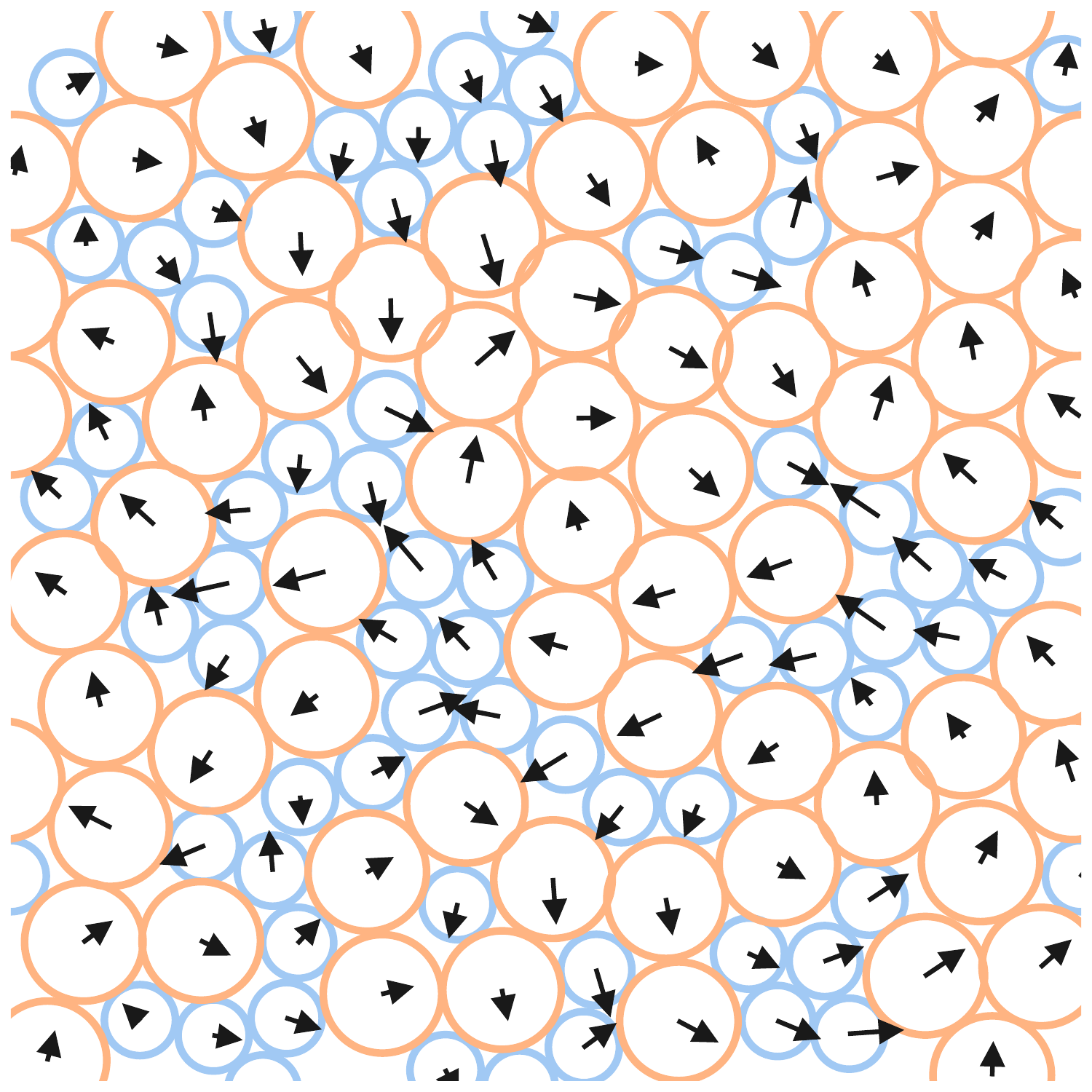}\\
    \caption{Particle positions and position sensitivity vectors, with respect
    to increasing the diameter of the blue particles.}
    \label{fig:jax-md-sensitivity}
\end{wrapfigure}

Many physical simulations require solving optimization problems,
such as energy minimization in molecular~\cite{jax-md} and
continuum~\cite{beatson2020composable} mechanics, structural
optimization~\cite{neural-reparam-2019} and data
assimilation~\cite{invobs_da2021}.  
In this experiment,
we revisit an example from JAX-MD~\cite{jax-md}, the problem of finding energy
minimizing configurations to a system of $k$ packed particles in a
$2$-dimensional box of size $\ell$
\begin{align}
x^\star(\theta)
&=
\argmin_{x \in \mathbb{R}^{k \times 2}}
f(x, \theta) \coloneqq \sum_{i,j} U(x_{i,j}, \theta),
\label{eq:jax-md}
\end{align}
where $x^\star(\theta) \in \RR^{k \times 2}$ are the optimal coordinates of the
$k$ particles, $U(x_{i,j}, \theta)$ is the pairwise potential energy function,
with half the particles at diameter 1 and half at diameter $\theta=0.6$, which
we optimize with a domain-specific optimizer~\cite{structural-relaxation-2006}.
Here we consider sensitivity of particle position with respect to diameter
$\partial x^\star(\theta)$, rather than sensitivity of the total energy from the
original experiment.  Figure~\ref{fig:jax-md-sensitivity} shows results
calculated via forward-mode implicit differentiation (JVP).  Whereas
differentiating the unrolled optimizer happens to work for total energy, here it
typically does not even converge, due to the discontinuous optimization method.

\section{Conclusion}

We proposed in this paper an approach for automatic implicit differentiation,
allowing the user to freely express the optimality conditions of the
optimization problem whose solutions are to be differentiated, 
directly in Python. The applicability of our approach to a large catalog of
optimality conditions is shown in the non-exhaustive list of Table \ref{tab:mapping_summary}, and illustrated by the ease with which
we can solve bi-level and sensitivity analysis problems. 






\clearpage

\appendix

\begin{center}
\huge \bfseries Appendix
\end{center}

\section{More examples of optimality criteria and fixed points}
\label{appendix:more_fixed_points}

To demonstrate the generality of our approach, we describe in this section more
optimality mapping $F$ or fixed point iteration $T$.

\paragraph{Mirror descent fixed point.}

We again consider the case when $x^\star(\theta)$ is implicitly defined as the
solution of \eqref{eq:constrained_pb}.
We now generalize the projected gradient fixed point beyond Euclidean geometry.
Let the Bregman divergence $D_\varphi \colon \dom(\varphi) \times
\relint(\dom(\varphi)) \to \RR_+$ generated by $\varphi$ be defined by
\begin{equation}
D_\varphi(x, y) \coloneqq \varphi(x) - \varphi(y)
- \langle \nabla \varphi(y), x - y \rangle.
\end{equation}
We define the Bregman projection of $y$ onto $\cC(\theta) \subseteq \dom(\varphi)$
by
\begin{equation}
\proj_\cC^\varphi(y, \theta) \coloneqq
\argmin_{x \in \cC(\theta)} D_\varphi(x, \nabla \varphi^*(y)).
\label{eq:Bregman_projection}
\end{equation}
Definition \eqref{eq:Bregman_projection} includes the mirror map $\nabla
\varphi^*(y)$ for convenience. It can be seen as a mapping from $\RR^d$ to
$\dom(\varphi)$, ensuring that \eqref{eq:Bregman_projection} is well-defined.
The mirror descent fixed point mapping is then
\begin{align}
    \hat x &= \nabla \varphi(x) \\
    y &= \hat x - \eta \nabla_1 f(x, \theta) \\
    T(x, \theta) &= \proj_\cC^\varphi(y, \theta).
\label{eq:mirror_descent_fp}
\end{align}
Because $T$ involves the composition of several functions,
manually deriving its JVP/VJP is error prone. This shows that our approach
leveraging autodiff
allows to handle more advanced fixed point mappings.
A common example of $\varphi$ is $\varphi(x) = \langle x, \log x - \ones \rangle$,
where $\dom(\varphi) = \RR_+^d$. In this case, $D_\varphi$ is the
Kullback-Leibler divergence. An advantage of the Kullback-Leibler projection is
that it sometimes easier to compute than the Euclidean projection, as we detail
in Appendix~\ref{appendix:jac_prod}.

\paragraph{Newton fixed point.}

Let $x$ be a root of $G(\cdot, \theta)$, i.e., $G(x, \theta) = 0$.
The fixed point iteration of Newton's method for root-finding is
\begin{equation}
T(x, \theta) = x - \eta [\partial_1 G(x, \theta)]^{-1} G(x, \theta).
\end{equation}
By the chain and product rules, we have
\begin{equation}
\partial_1 T(x, \theta) = I 
- \eta (...) G(x, \theta) 
- \eta [\partial_1 G(x, \theta)]^{-1} \partial_1 G(x, \theta)
= (1 - \eta) I.
\end{equation}
Using \eqref{eq:fixed_point_to_root}, we get
$A = -\partial_1 F(x, \theta) = \eta I$.
Similarly,
\begin{equation}
B = \partial_2 T(x, \theta) = \partial_2 F(x, \theta) = 
- \eta [\partial_1 G(x, \theta)]^{-1} \partial_2 G(x, \theta).
\end{equation}
Newton's method for optimization is obtained by choosing 
$G(x, \theta) = \nabla_1 f(x, \theta)$,
which gives
\begin{equation}
T(x, \theta) = x - \eta [\nabla_1^2 f(x, \theta)]^{-1} \nabla_1 f(x, \theta).
\label{eq:newton_opt_fp}
\end{equation}
It is easy to check that we
recover the same linear system as for the gradient descent fixed point
\eqref{eq:gradient_descent_fp}.
A practical implementation can pre-compute an LU decomposition of $\partial_1
G(x, \theta)$, or a Cholesky decomposition if $\partial_1 G(x, \theta)$
is positive semi-definite.

\paragraph{Proximal block coordinate descent fixed point.}

We now consider the case when $x^\star(\theta)$ is implicitly defined as the
solution
\begin{equation}
x^\star(\theta) \coloneqq \argmin_{x \in \RR^d} f(x, \theta) + 
\sum_{i=1}^m g_i(x_i, \theta),
\end{equation}
where $g_1, \dots, g_m$ are possibly non-smooth functions operating on
subvectors (blocks) $x_1, \dots, x_m$ of $x$. In this case, we can use for $i
\in [m]$ the fixed point
\begin{equation}
x_i = [T(x, \theta)]_i = 
\prox_{\eta_i g_i}(x_i - \eta_i [\nabla_1 f(x, \theta)]_i, \theta),
\label{eq:bcd_fp}
\end{equation}
where $\eta_1, \dots, \eta_m$ are block-wise step sizes. Clearly, when the step
sizes are shared, i.e., $\eta_1 = \dots = \eta_m = \eta$, this fixed point is
equivalent to the proximal gradient fixed point \eqref{eq:proximal_grad_fp}
with $g(x, \theta) = \sum_{i=1}^n g_i(x_i, \theta)$.

\paragraph{Quadratic programming.}

We now show how to use the KKT conditions discussed in
\S\ref{sec:mapping_examples} to differentiate quadratic programs,
recovering Optnet \cite{amos_2017} as a special case.  To give some
intuition, let us start with a simple equality-constrained quadratic program
(QP)
\begin{equation}
\argmin_{z \in \RR^p} 
f(z, \theta) = \frac{1}{2} z^\top Q z + c^\top z
\quad \text{subject to} \quad
H(z, \theta) = E z - d = 0,
\end{equation}
where $Q \in \RR^{p \times p}$, $E \in \RR^{q \times p}$,
$d \in \RR^q$. 
We gather the differentiable parameters as $\theta = (Q, E, c, d)$.
The stationarity and primal feasibility conditions give
\begin{align}
\nabla_1 f(z, \theta) + [\partial_1 H(z, \theta)]^\top \nu
= Q z + c + E^\top \nu &= 0 \\
H(z, \theta) = E z - d &= 0.
\end{align}
In matrix notation, this can be rewritten as
\begin{equation}
\begin{bmatrix}
   Q & E^\top \\
   E & 0
\end{bmatrix}
\begin{bmatrix} 
z \\ 
\nu 
\end{bmatrix}
=
\begin{bmatrix} 
-c \\ 
d
\end{bmatrix}.
\label{eq:eq_const_qp}
\end{equation}
We can write the solution of the linear system
\eqref{eq:eq_const_qp} as the root $x = (z, \nu)$ of a function $F(x,
\theta)$. More generally, the QP can also include inequality constraints
\begin{equation}
\argmin_{z \in \RR^p} 
f(z, \theta) = \frac{1}{2} z^\top Q z + c^\top z
\quad \text{subject to} \quad
H(z, \theta) = E z - d = 0, 
G(z, \theta) = M z - h \le 0.
\label{eq:qp_eq_ineq}
\end{equation}
where $M \in \RR^{r \times p}$ and $h \in \RR^r$.
We gather the differentiable parameters as $\theta = (Q, E, M, c, d, h)$.
The stationarity, primal feasibility and complementary slackness conditions give
\begin{align}
\nabla_1 f(z, \theta) + [\partial_1 H(z, \theta)]^\top \nu
+ [\partial_1 G(z, \theta)]^\top \lambda =
Q z + c + E^\top \nu + M^\top \lambda &= 0 \\
H(z, \theta) = E z - d &= 0 \\
\lambda \circ G(z, \theta) = \diag(\lambda)(Mz-h) &= 0
\end{align}
In matrix notation, this can be written as
\begin{equation}
\begin{bmatrix}
    Q & E^\top & M^\top \\
    E & 0 & 0 \\
    \diag(\lambda)M & 0 & 0
\end{bmatrix}
\begin{bmatrix} 
z \\ 
\nu \\
\lambda
\end{bmatrix}
=
\begin{bmatrix} 
-c \\ 
d \\
\lambda \circ h
\end{bmatrix}
\label{eq:ineq_const_qp}
\end{equation}
While $x = (z, \nu, \lambda)$ is no longer the solution of a linear system, it
is the root of a function $F(x, \theta)$ and therefore fits our framework.
With our framework, no derivation is needed. We simply define
$f$, $H$ and $G$ directly in Python.

\paragraph{Conic programming.}

We now show that the differentiation of conic linear programs
\cite{agrawal_cone,amos_thesis}, at the heart of differentiating through cvxpy
layers \cite{agrawal_2019}, easily fits our framework.
Consider the problem
\begin{equation}
z^\star(\lambda), s^\star(\lambda) =
\argmin_{z \in \RR^p, s \in \RR^m} c^\top z 
\quad \text{subject to} \quad
Ez + s = d, s \in \cK,
\label{eq:linear_conic_primal}
\end{equation}
where $\lambda = (c, E, d)$, 
$E \in \RR^{m \times p}$, $d \in \RR^m$, $c \in \RR^p$ and $\cK \subseteq
\RR^m$ is a cone;
$z$ and $s$ are the primal and slack variables, respectively.
Every convex optimization problem can be reduced to the form
\eqref{eq:linear_conic_primal}.
Let us form the skew-symmetric matrix
\begin{equation}
\theta(\lambda) =
\begin{bmatrix}
    0 & E^\top & c \\
    -E & 0 & d \\
    -c^\top & -d^\top & 0
\end{bmatrix}
\in \RR^{N \times N},
\end{equation}
where $N = p + m + 1$.
Following \cite{agrawal_cone,agrawal_2019,amos_thesis},
we can use the homogeneous self-dual embedding 
to reduce the process of solving \eqref{eq:linear_conic_primal} to finding
a root of the residual map
\begin{equation}
F(x, \theta) 
= \theta \Pi x + \Pi^* x
= ((\theta - I) \Pi + I) x,
\label{eq:residual_map}
\end{equation}
where $\Pi = \proj_{\RR^p \times \cK^* \times \RR_+}$
and $\cK^* \subseteq \RR^m$ is the dual cone. 
The splitting conic solver \cite{splitting_conic}, which is based on 
ADMM, outputs a solution 
$F(x^\star(\theta), \theta) = 0$ which is decomposed as
$x^\star(\theta) = (u^\star(\theta), v^\star(\theta), 
w^\star(\theta))$. We can then recover the optimal solution of
\eqref{eq:linear_conic_primal} using
\begin{equation}
z^\star(\lambda) = u^\star(\theta(\lambda))
\quad \text{and} \quad
s^\star(\lambda) = \proj_{\cK^*}(v^\star(\theta(\lambda))) -
v^\star(\theta(\lambda)).
\end{equation}
The key oracle whose JVP/VJP we need 
is therefore $\Pi$, which is studied in \cite{ali_2017}.
The projection onto a few cones is available in our library
and can be used to express $F$.

\paragraph{Frank-Wolfe.}

We now consider
\begin{equation}
x^\star(\theta) = \argmin_{x \in \cC(\theta) \subset \RR^d} f(x, \theta),
\label{eq:fw_pb}
\end{equation}
where $\cC(\theta)$ is a convex polytope, i.e., it is the convex hull of
vertices $v_1(\theta), \dots, v_m(\theta)$. The Frank-Wolfe algorithm requires a
linear minimization oracle (LMO)
\begin{equation*}
    s \mapsto \argmin_{x \in \cC(\theta)} ~\langle s, x \rangle
\end{equation*}
and is a popular algorithm when this LMO is easier to compute than
the projection onto $\cC(\theta)$. However, since this LMO is piecewise constant,
its Jacobian is null almost everywhere. Inspired by SparseMAP \cite{sparsemap},
which corresponds to the case when $f$ is a quadratic,
we rewrite \eqref{eq:fw_pb} as
\begin{equation}
p^\star(\theta) = \argmin_{p \in \triangle^m}
g(p, \theta) \coloneqq f(V(\theta) p, \theta),
\end{equation}
where $V(\theta)$ is a $d \times m$ matrix gathering the 
vertices $v_1(\theta), \dots, v_m(\theta)$.
We then have $x^\star(\theta) = V(\theta) p^\star(\theta)$.
Since we have reduced \eqref{eq:fw_pb} to minimization over the simplex,
we can use the projected gradient fixed point to obtain
\begin{equation}
T(p^\star(\theta), \theta) = 
\proj_{\triangle^m}(p^\star(\theta) - \nabla_1 g(p^*(\theta), \theta)).
\end{equation}
We can therefore compute the derivatives of $p^\star(\theta)$ by implicit
differentiation and the derivatives of $x^\star(\theta)$ by product rule.
Frank-Wolfe implementations typically maintain the 
convex weights of the vertices, which we use to get an approximation of
$p^\star(\theta)$. Moreover, it is well-known that after $t$
iterations, at most $t$ vertices are visited. We can leverage this sparsity to
solve a smaller linear system. Moreover, in practice,
we only need to compute VJPs of $x^\star(\theta)$.

\section{Code examples}
\label{appendix:code_examples}

\subsection{Code examples for optimality conditions}
\label{appendix:opt_cond_mapping}

Our library provides several reusable optimality condition mappings $F$ or fixed
points $T$. We nevertheless demonstrate the ease of writing some of
them from scratch.

\paragraph{KKT conditions.}

As a more advanced example, we now describe how to implement the KKT conditions
\eqref{eq:kkt_conditions}. The
stationarity, primal feasibility and complementary slackness conditions read
\begin{align}
\nabla_1 f(z, \theta_f) + [\partial_1 G(z, \theta_G)]^\top \lambda + 
[\partial_1 H(z, \theta_H)]^\top \nu = 0 \\
H(z, \theta_H) = 0 \\
\lambda \circ G(z, \theta_G) = 0.
\end{align}
Using \texttt{jax.vjp} to compute vector-Jacobian products, 
this can be implemented as
\begin{figure}[H]
\centering
\fbox{
\includegraphics{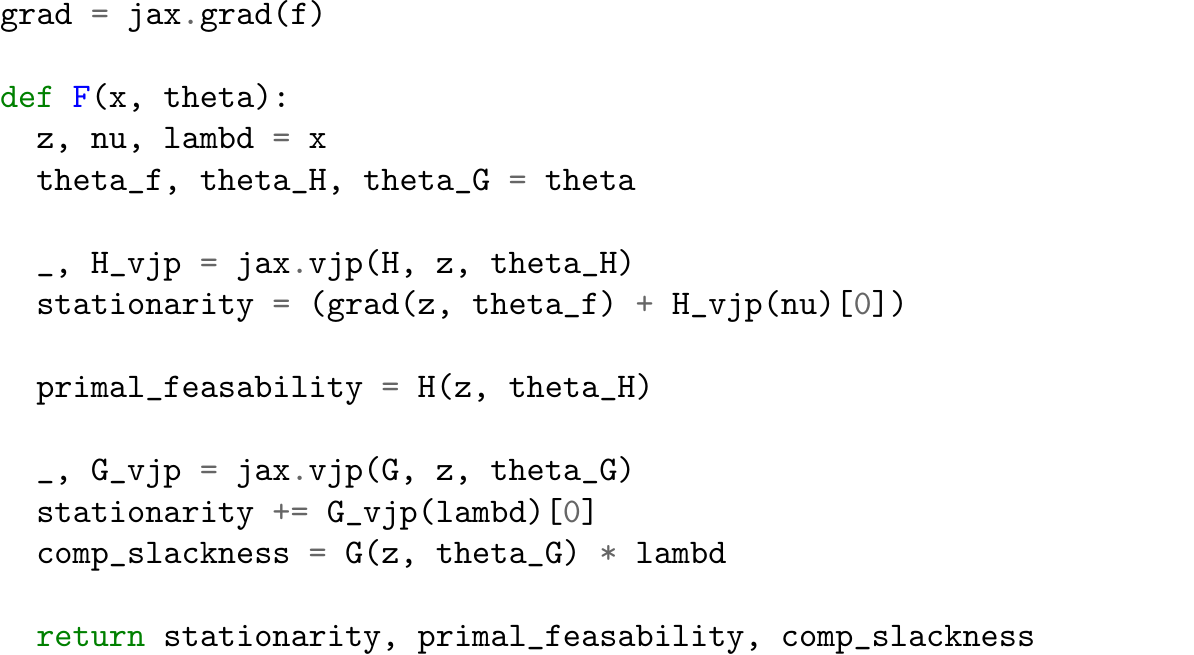}
}
\caption{KKT conditions $F(x, \theta)$}
\label{fig:kkt_code}
\end{figure}
Similar mappings $F$ can be written if the optimization problem
contains only equality constraints or only inequality constraints.




\paragraph{Mirror descent fixed point.}

Letting $\eta=1$ and
denoting $\theta = (\theta_f, \theta_{\proj})$,
the fixed point \eqref{eq:mirror_descent_fp} is
\begin{align}
    \hat x &= \nabla \varphi(x) \\
    y &= \hat x - \nabla_1 f(x, \theta_f) \\
    T(x, \theta) &= \proj_\cC^\varphi(y, \theta_{\proj}).
\end{align}
We can then implement it as follows.
\begin{figure}[H]
\centering
\fbox{
\includegraphics{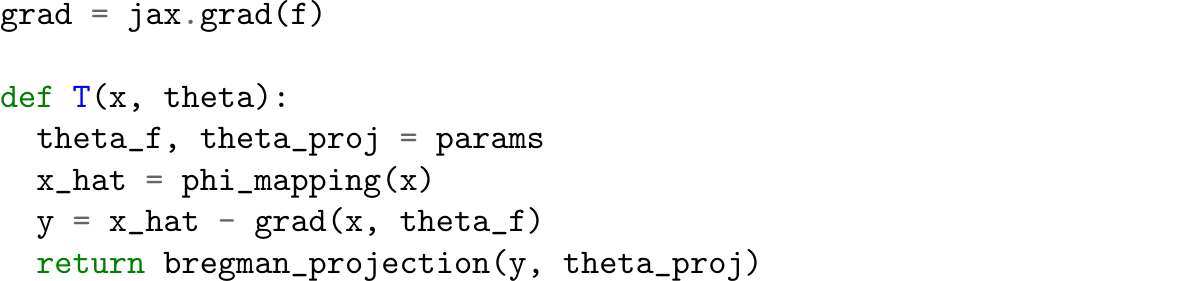}
}
\caption{Mirror descent fixed point $T(x, \theta)$}
\label{fig:mirror_descent_code}
\end{figure}
Although not considered in this example, the mapping $\nabla \varphi$ could also
depend on $\theta$ if necessary.

\subsection{Code examples for experiments}

We now sketch how to implement our experiments using our framework.
In the following, \texttt{jnp} is short for \texttt{jax.numpy}.
In all experiments, we only show how to compute gradients
with the outer objective. We can then use these gradients with gradient-based
solvers to solve the outer objective.

\paragraph{Multiclass SVM experiment.}

~

\begin{figure}[H]
\centering
\fbox{
    \includegraphics[scale=0.9]{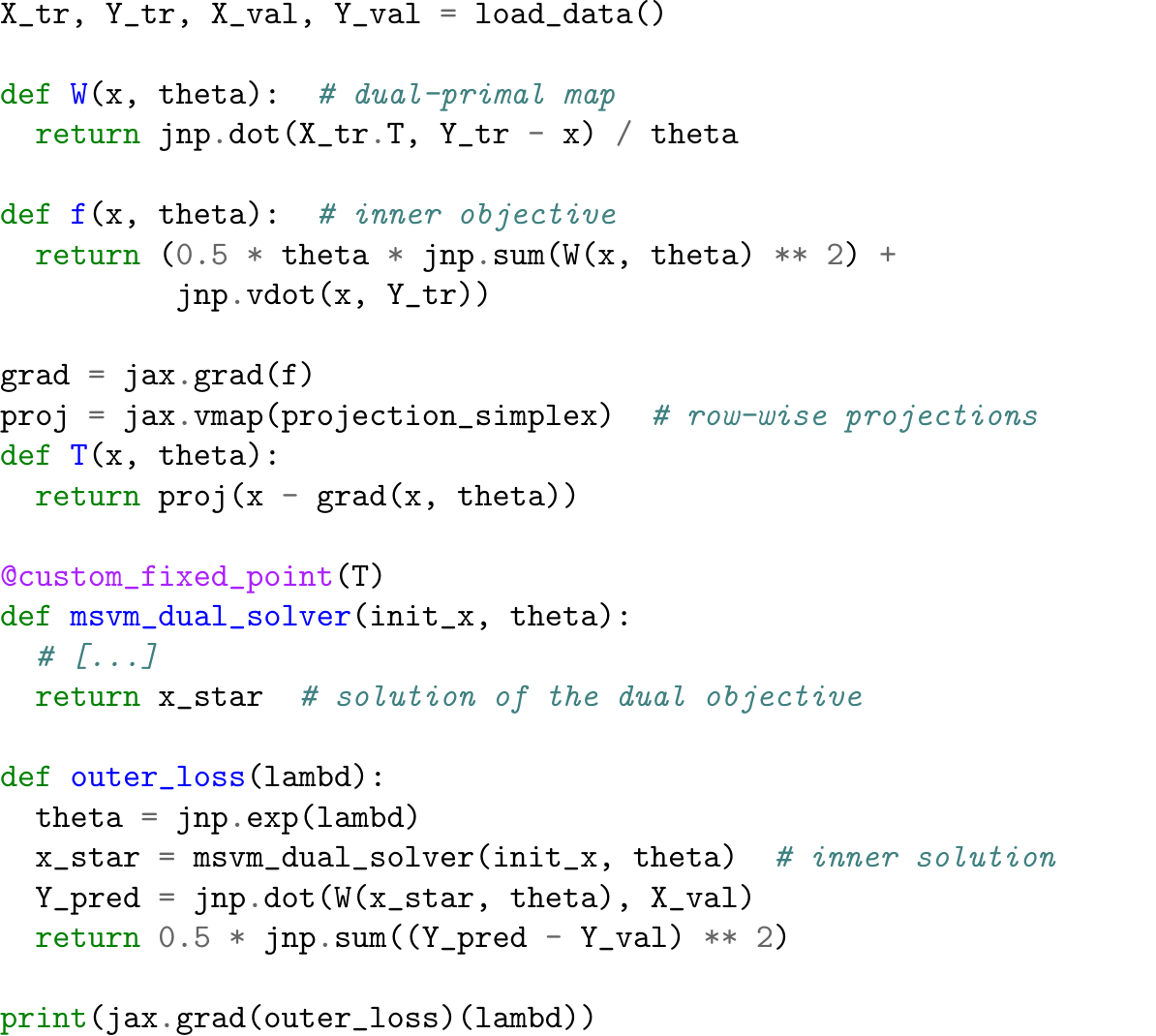}
}
\caption{Code example for the multiclass SVM experiment.}
\label{fig:msvm_code}
\end{figure}

\paragraph{Task-driven dictionary learning experiment.}

~

\begin{figure}[H]
\centering
\fbox{
\includegraphics{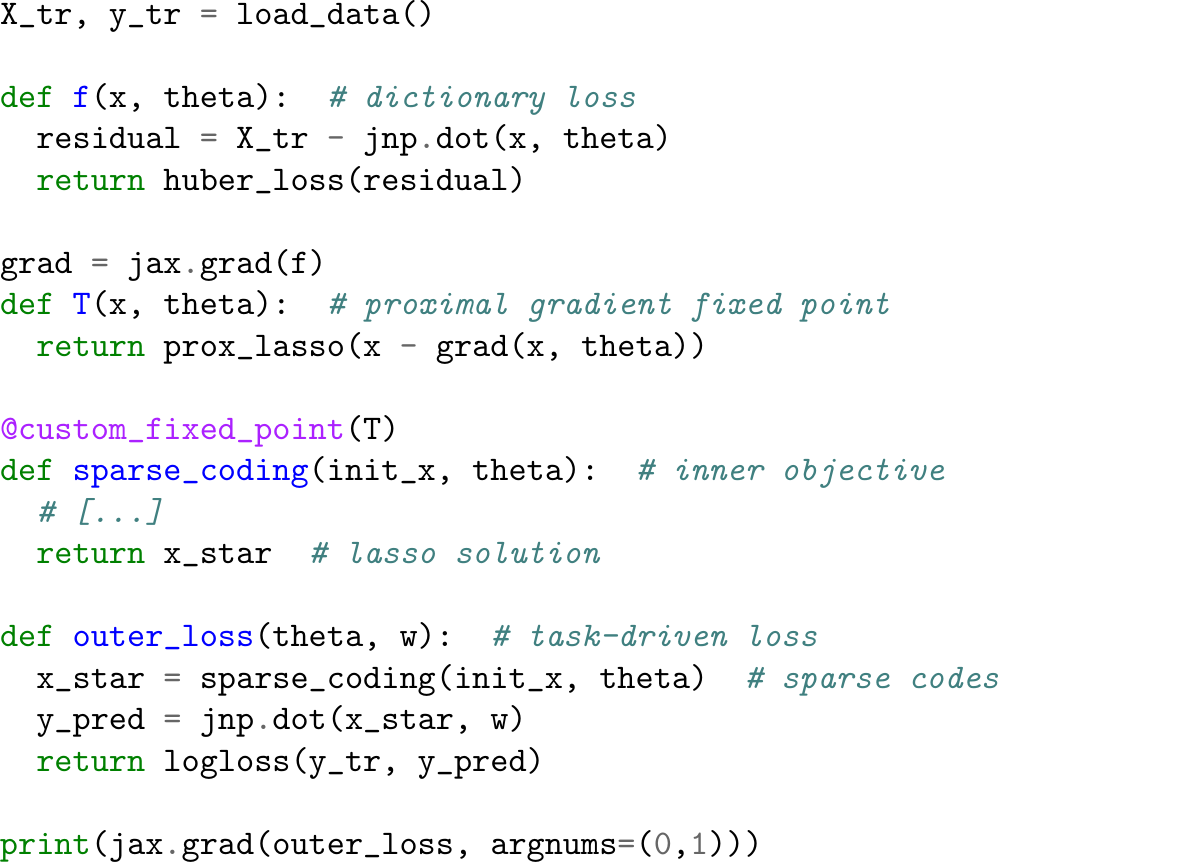}
}
\caption{Code example for the task-driven dictionary learning experiment.}
\label{fig:sparse_coding_code}
\end{figure}

\paragraph{Dataset distillation experiment.}

~

\begin{figure}[H]
\centering
\fbox{
\includegraphics{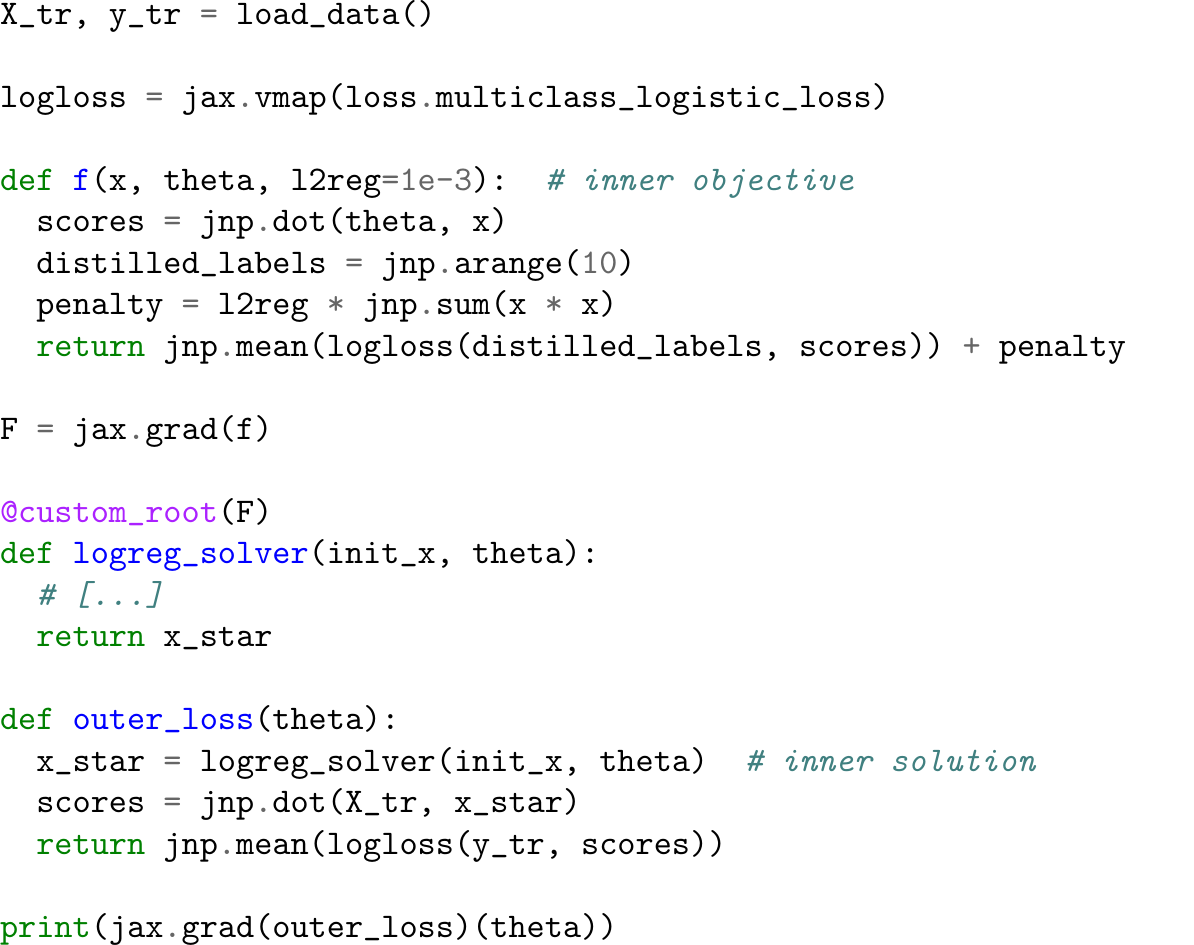}
}
\caption{Code example for the dataset distillation experiment.}
\label{fig:distillation_code}
\end{figure}

\paragraph{Molecular dynamics experiment.}

~

\begin{figure}[H]
\centering
\fbox{
\includegraphics{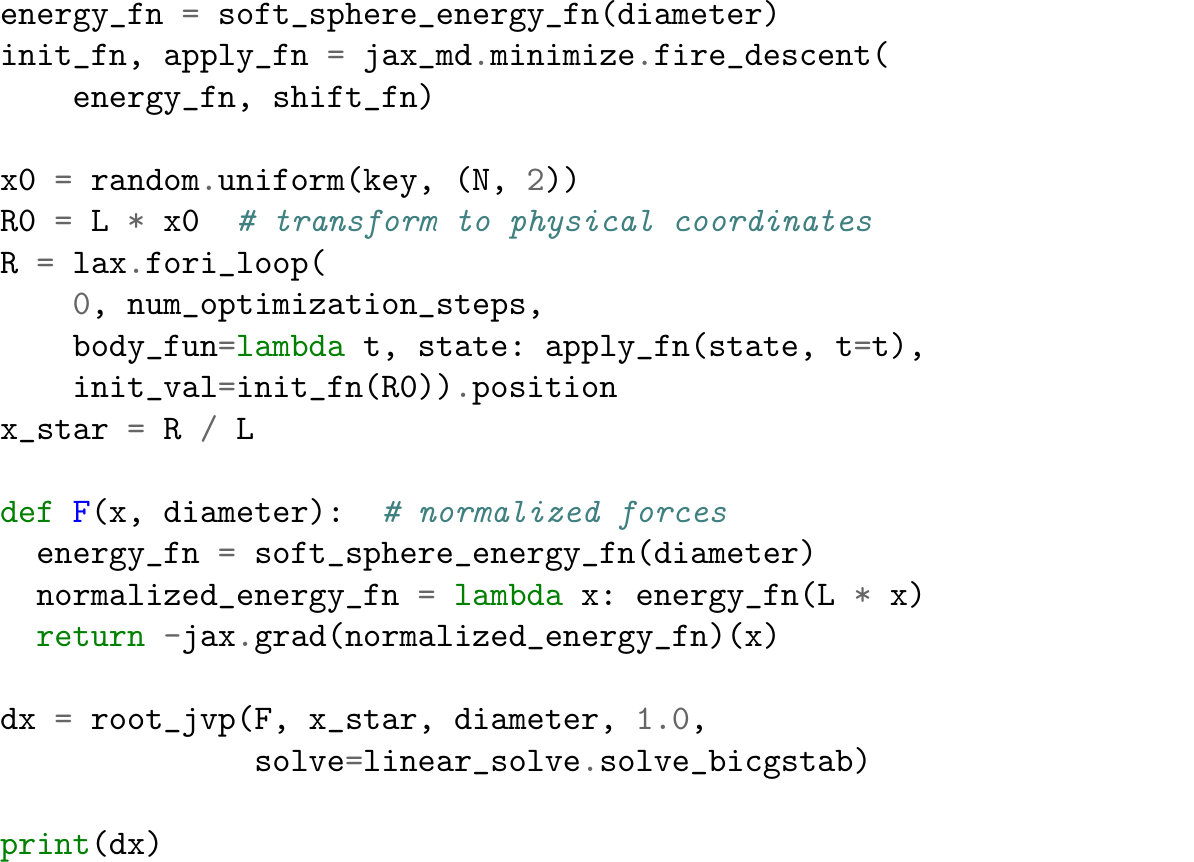}
}
\caption{Code for the molecular dynamics experiment.}
\label{fig:md_code}
\end{figure}

\section{Jacobian products}
\label{appendix:jac_prod}

Our library provides numerous reusable building blocks.
We describe in this section how to compute their Jacobian products.
As a general guideline, whenever a projection enjoys a closed form, we leave the
Jacobian product to the autodiff system.  

\subsection{Jacobian products of projections}
\label{sec:jacobians_proj}

We describe in this section how to compute the Jacobian products of the
projections (in the Euclidean and KL senses) onto various convex sets.  
When the convex set does not depend on any variable, we simply denote
it $\cC$ instead of $\cC(\theta)$.

\paragraph{Non-negative orthant.}

When $\cC$ is the non-negative orthant, $\cC = \RR^d_+$, we obtain
$\proj_\cC(y) = \max(y, 0)$,
where the maximum is evaluated element-wise. This is also known as the ReLu
function. The projection in the KL sense reduces to the exponential function,
$\proj^\varphi_\cC(y) = \exp(y)$.

\paragraph{Box constraints.}

When $\cC(\theta)$ is the box constraints $\cC(\theta) = [\theta_1, \theta_2]^d$ 
with $\theta \in \RR^2$, we obtain
\begin{equation}
\proj_{\cC}(y, \theta) = \text{clip}(y, \theta_1, \theta_2) \coloneqq
\max( \min(y, \theta_2), \theta_1).
\end{equation}
This is trivially extended to support different boxes for each coordinate, 
in which case $\theta \in \RR^{d \times 2}$.

\paragraph{Probability simplex.}

When $\cC$ is the standard probability simplex,
$\cC = \triangle^d$, there is no analytical solution for $\proj_\cC(y)$. 
Nevertheless, the projection can be computed exactly in $O(d)$ expected
time or $O(d \log d)$ worst-case time
\cite{Brucker1984,michelot,duchi,Condat2016}.
The Jacobian is given by $\text{diag}(s) - s s^\top / \|s\|_1$, 
where $s \in \{0,1\}^d$ is a vector indicating the support of $\proj_\cC(y)$
\cite{sparsemax}. The projection in the KL sense, on the other hand, enjoys a
closed form: it reduces to the usual softmax
$\proj^\varphi_\cC(y) = \exp(y) / \sum_{j=1}^d \exp(y_j)$.

\paragraph{Box sections.}

Consider now the Euclidean projection 
$z^\star(\theta) = \proj_{\cC}(y, \theta)$ onto the set 
$\cC(\theta) = \{z \in \RR^d \colon \alpha_i \le z_i \le \beta_i, i \in
[d]; w^\top z = c\}$, where $\theta = (\alpha, \beta, w, c)$.
This projection is a singly-constrained bounded quadratic program.
It is easy to check (see, e.g., \cite{lp_sparsemap}) that an optimal solution
satisfies for all $i \in [d]$
\begin{equation}
z^\star_i(\theta) = [L(x^\star(\theta), \theta)]_i
\coloneqq \text{clip}(w_i x^\star(\theta) + y_i, \alpha_i, \beta_i)
\end{equation}
where $L \colon \RR \times \RR^n \to \RR^d$ is the dual-primal mapping
and $x^\star(\theta) \in \RR$ is the optimal dual variable of the linear
constraint, which should be the root of
\begin{equation}
F(x^\star(\theta), \theta) = L(x^\star(\theta), \theta)^\top w - c.
\end{equation}
The root can be found, e.g., by bisection.
The gradient $\nabla x^\star(\theta)$ is given by
$\nabla x^\star(\theta) = B^\top / A$
and the Jacobian $\partial z^\star(\theta)$ is
obtained by application of the chain rule on $L$.

\paragraph{Norm balls.}

When $\cC(\theta) = \{ x \in \RR^d \colon \|x\| \le \theta\}$, where $\|\cdot\|$
is a norm and $\theta \in \RR_+$, $\proj_\cC(y, \theta)$ becomes the projection
onto a norm ball.  The projection onto the $\ell_1$-ball reduces to a projection
onto the simplex, see, e.g., \cite{duchi}. The projections onto the $\ell_2$ and
$\ell_\infty$ balls enjoy a closed-form, see, e.g., \cite[\S 6.5]{parikh_2014}.
Since they rely on simple composition of functions, all three projections can
therefore be automatically differentiated.

\paragraph{Affine sets.}

When $\cC(\theta) = \{x \in \RR^d \colon A x = b\}$,
where $A \in \RR^{p \times d}$, $b \in \RR^p$ and $\theta = (A, b)$,
we get
\begin{equation}
\proj_{\cC}(y, \theta) = 
y - A^\dagger(Ay - b)
= y - A^\top (AA^\top)^{-1}(A y - b)
\end{equation}
where $A^\dagger$ is the Moore-Penrose pseudoinverse of $A$.
The second equality holds if $p < d$ and $A$ is full rank.
A practical implementation can pre-compute a factorization of the Gram matrix $A
A^\top$. Alternatively, we can also use the KKT conditions.

\paragraph{Hyperplanes and half spaces.}

When $\cC(\theta) = \{x \in \RR^d \colon a^\top x = b\}$, 
where $a \in \RR^d$ and $b \in \RR$ and $\theta = (a, b)$, 
we get
\begin{equation}
\proj_{\cC}(y, \theta) = 
y - \frac{a^\top y - b}{\|a\|_2^2} a.
\end{equation}
When $\cC(\theta) = \{x \in \RR^d \colon a^\top x \le b\}$, we simply replace
$a^\top y - b$ in the numerator by $\max(a^\top y - b, 0)$.

\paragraph{Transportation and Birkhoff polytopes.}

When $\cC(\theta) = \{X \in \RR^{p \times d} \colon X \ones_d = \theta_1, X^\top
\ones_p = \theta_2, X \ge 0\}$, 
the so-called transportation polytope, 
where $\theta_1 \in \triangle^p$ and $\theta_2 \in \triangle^d$ are marginals,
we can compute approximately the projections, both in the Euclidean and KL
senses, by switching to the dual or semi-dual \cite{blondel_2018_ot}. Since both
are unconstrained optimization problems, we can compute their Jacobian product
by implicit differentiation using the gradient descent fixed point. 
An advantage of the KL geometry here is that we can
use Sinkhorn \cite{cuturi_2013_sinkhorn}, which is a GPU-friendly algorithm.
The Birkhoff polytope, the set of doubly stochastic matrices, is obtained by
fixing $\theta_1 = \theta_2 = \ones_d / d$.

\paragraph{Order simplex.}

When $\cC(\theta) = \{x \in \RR^d \colon \theta_1 \ge x_1 \ge x_2 \ge \dots \ge
x_d \ge \theta_2\}$, a so-called order simplex
\cite{grotzinger_1984,blondel_2019_structured}, the projection operations, both
in the Euclidean and KL sense, reduce to isotonic optimization \cite{lim_2016}
and can be solved exactly in $O(d \log d)$ time using the Pool Adjacent
Violators algorithm \cite{best_2000}. The Jacobian of the projections and
efficient product with it are derived in \cite{djolonga_2017,blondel_2020_fast}.

\paragraph{Polyhedra.}

More generally, we can consider polyhedra, i.e., sets of the form
$\cC(\theta) = \{x \in \RR^d \colon A x = b, Cx \le d\}$,
where $A \in \RR^{p \times d}$, $b \in \RR^p$, $C \in \RR^{m \times d}$,
and $d \in \RR^m$. 
There are several ways to differentiate this projection.
The first is to use the KKT conditions as detailed in
\S\ref{sec:mapping_examples}.
A second way is consider the dual of the projection instead, 
which is the maximization of a quadratic function subject to
\textbf{non-negative constraints} \cite[\S 6.2]{parikh_2014}. 
That is, we can reduce the projection on a polyhedron
to a problem of the form \eqref{eq:constrained_pb} with non-negative constraints, 
which we can in turn
implicitly differentiate easily using the projected gradient fixed point,
combined with the projection on the non-negative orthant.
Finally, we apply the dual-primal mapping ,
which enjoys a closed form and is therefore amenable to autodiff,
to obtain the primal projection.

\subsection{Jacobian products of proximity operators}

We provide several proximity operators, including for the lasso (soft
thresholding), elastic net and group lasso (block soft thresholding).
All satisfy closed form expressions and can be differentiated automatically via
autodiff.
For more advanced proximity operators, which do not enjoy a closed form,
recent works have derived their Jacobians. 
The Jacobians of fused lasso and OSCAR were derived in
\cite{niculae_2017}.  For general total variation, the Jacobians were derived in
\cite{vaiter_2013,cherkaoui_2020}.

\section{Jacobian precision proofs}
\label{appendix:proofs}

\begin{proof}[Proof of Theorem~\ref{thm:jacob}]
To simplify notations, we note $A_\star \coloneqq A(x^\star, \theta)$ and $\hat
A \coloneqq
A(\hat x, \theta)$, and similarly for $B$ and $J$. We have by definition of the
Jacobian estimate function $A_\star J_\star = B_\star$ and $\hat A \hat J = \hat
B$. Therefore we have
    \begin{align}
        J(\hat x, \theta) - \partial x^\star (\theta) &= \hat A^{-1} \hat B - A_\star^{-1} B_\star \\
        &= \hat A^{-1} \hat B - \hat A^{-1} B_\star + \hat A^{-1} B_\star - A_\star^{-1} B_\star \\
        &= \hat A^{-1}( \hat B -  B_\star) + (\hat A^{-1} - A_\star^{-1}) B_\star\, .
    \end{align}
For any invertible matrices $M_1, M_2$, it holds that $M_1^{-1} - M_2^{-1} = M_1^{-1} (M_2 - M_1) M_2^{-1}$, so 
\begin{equation}\label{eq:normproduct}
    \|M_2^{-1} - M_2^{-1}\|_{\text{op}} \leq \|M_1^{-1}\|_{\text{op}} \|M_2 -
    M_1\|_{\text{op}} \|M_2^{-1}\|_{\text{op}} \,.
\end{equation}
Therefore,
\begin{equation}\label{eq:lipinv}
    \|\hat A^{-1} - A_\star^{-1}\|_{\text{op}} \le \frac{1}{\alpha^2} \|\hat A -
    A_\star\|_{\text{op}}   \le \frac{\gamma}{\alpha^2}  \|\hat x - x^\star(\theta)\|\, .
\end{equation}
As a consequence, the second term in $J(\hat x, \theta) - \partial x^\star (\theta)$ can be upper bounded and we obtain
    \begin{align}
        \|J(\hat x, \theta) - \partial x^\star (\theta)\| &\le \|\hat A^{-1}( \hat B -  B_\star)\| + \|(\hat A^{-1} - A_\star^{-1}) B_\star\| \\
        &\le \|\hat A^{-1} \|_{\text{op}}\| \hat B -  B_\star\| +\frac{\gamma}{\alpha^2} \|\hat x - x^\star(\theta)\|\ \|B_\star\|  \, ,
    \end{align}
which yields the desired result.
\end{proof}

\begin{corollary}[Jacobian precision for gradient descent fixed point]
\label{cor:precision-gd}
    Let $f$ be such that $f(\cdot, \theta)$ is twice differentiable and
    $\alpha$-strongly convex and $\nabla_1^2f(\cdot, \theta)$ is
    $\gamma$-Lipschitz (in the operator norm) and $\partial_2\nabla_1 f(x, \theta)$ is $\beta$-Lipschitz and bounded in norm by $R$. The estimated Jacobian evaluated at $\hat x$ is then given by 
 \[   
 J(\hat x, \theta) = - (\nabla^2_1f(\hat x, \theta))^{-1} \partial_2\nabla_{1} f(\hat x, \theta)\, .
 \]
 For all $\theta \in \RR^n$, and any $\hat{x}$ estimating $x^\star(\theta)$, we have the following bound for the approximation error of the estimated Jacobian
\[
\|J(\hat x, \theta) - \partial x^\star (\theta)\| \le \left(\frac{\beta}{\alpha} + \frac{\gamma R}{\alpha^2}\right) \|\hat x - x^\star(\theta)\|\, .
\]
\end{corollary}
\begin{proof}[Proof of Corollary~\ref{cor:precision-gd}]
This follows from Theorem~\ref{thm:jacob}, applied to this specific $A(x, \theta)$ and $B(x, \theta)$. 
\end{proof}

For proximal gradient descent, where 
$T(x, \theta) = \prox_{\eta g}(x - \eta \nabla_1 f(x, \theta), \theta)$,
this yields
\begin{align*}
A(x, \theta) &= I - \partial_1 T(x, \theta) = 
I - 
\partial_1 \prox_{\eta g}(x - \eta \nabla_1 f(x, \theta), \theta )
(I - \eta \nabla_1^2 f(x, \theta)) 
\\
B(x, \theta) &= 
\partial_2 \prox_{\eta g}(x - \eta \nabla_1 f(x, \theta) , \theta) -
\eta 
\partial_1 \prox_{\eta g}(x - \eta \nabla_1 f(x, \theta) , \theta)
\partial_2 \nabla_1 f(x, \theta) 
\, .
\end{align*}

We now focus in the case of proximal gradient descent on an objective
$f(x,\theta) + g(x)$, where $g$ is smooth and does not depend
on $\theta$. This is the case in our experiments in \S\ref{sec:sparse-coding}.
Recent work also exploits local smoothness of solutions to derive similar bounds
\cite[Theorem 13]{bertrand_2021_journal}

\begin{corollary}[Jacobian precision for proximal gradient descent fixed point]
\label{cor:precision-prox}
    Let $f$ be such that $f(\cdot, \theta)$ is twice differentiable and
    $\alpha$-strongly convex and $\nabla_1^2f(\cdot, \theta)$ is
    $\gamma$-Lipschitz (in the operator norm) and $\partial_2\nabla_1 f(x,
    \theta)$ is $\beta$-Lipschitz and bounded in norm by $R$. Let $g: \RR^d \to
    \RR$ be a twice-differentiable $\mu$-strongly convex (with special case $\mu =0$ being only convex), for which the function
    $\Gamma_\eta(x, \theta) =  \nabla^2 g(\prox_{\eta g}(x - \eta \nabla_1 f(x,
    \theta))$ is $\kappa_\eta$-Lipschitz in it first argument. The estimated Jacobian
    evaluated at $\hat x$ is then given by 
 \[   
 J(\hat x, \theta) = - (\nabla^2_1f(\hat x,
 \theta) + \Gamma_\eta(\hat x, \theta))^{-1} \partial_2\nabla_{1} f(\hat x, \theta)\, .
 \]
 For all $\theta \in \RR^n$, and any $\hat{x}$ estimating $x^\star(\theta)$, we have the following bound for the approximation error of the estimated Jacobian
\[
\|J(\hat x, \theta) - \partial x^\star (\theta)\| \le \left(\frac{\beta + \kappa_\eta}{\alpha + \mu} + \frac{\gamma R}{(\alpha + \mu)^2}\right) \|\hat x - x^\star(\theta)\|\, .
\]
\end{corollary}
\begin{proof}[Proof of Corollary~\ref{cor:precision-prox}]
First, let us note that $\prox_{\eta g} (y, \theta)$ does not depend on $\theta$,  since $g$ itself does not depend on $\theta$, and is therefore equal to classical proximity operator of $\eta g$ which, with a slight overload of notations, we denote as $\prox_{\eta g}(y)$ (with a single argument). In other words,
\begin{equation*}
    \begin{cases}
        \prox_{\eta g} (y, \theta) &= \prox_{\eta g} (y)\,,\\
        \partial_1 \prox_{\eta g} (y, \theta) &= \partial\prox_{\eta g} (y)\,,\\
        \partial_2 \prox_{\eta g} (y, \theta) &= 0\,.
    \end{cases}
\end{equation*}

Regarding the first claim (expression of the estimated Jacobian evaluated at $\hat x$), we first have that $\prox_{\eta g} (y)$
 is the solution to $(x' - y) + \eta \nabla g(x') = 0$ in $x'$ - by first-order condition for a smooth convex function. We therefore have that
\begin{align*}
\prox_{\eta g}(y) &= (I + \eta \nabla g)^{-1} (y)\\
\partial \prox_{\eta g}(y) &= (I_d + \eta \nabla^2 g(\prox_{\eta g}(y)))^{-1}\, ,
\end{align*}
the first $I$ and inverse being functional identity and inverse, and the second $I_d$ and inverse being in the matrix sense, by inverse rule for Jacobians $\partial h(z) = [\partial h^{-1}(h(z))]^{-1}$ (applied to the prox).

As a consequence, we have, for $\Gamma_\eta(x, \theta) =  \nabla^2 g(\prox_{\eta g}(x - \eta \nabla_1 f(x, \theta))$ that 
\begin{align*}
A(x, \theta) &= I_d - (I_d + \eta \Gamma_\eta(x, \theta))^{-1}(I_d - \eta \nabla_1^2 f(x, \theta)) \\
&= (I_d + \eta \Gamma_\eta(x, \theta))^{-1} [I_d + \eta \Gamma_\eta(x, \theta) - (I_d - \eta \nabla_1^2 f(x, \theta))] \\
&= \eta (I_d + \eta \Gamma_\eta(x, \theta))^{-1} ( \nabla^2_1 f(x, \theta) +  \Gamma_\eta(x, \theta)) \\
B(x, \theta) &= - \eta  (I_d + \eta \Gamma_\eta(x, \theta))^{-1} \partial_2 \nabla_1 f(x, \theta)\, .
\end{align*}
As a consequence, for all $x \in \RR^d$, we have that
\[
J(x, \theta) = - ( \nabla^2_1 f(x, \theta) +
\Gamma_\eta(x, \theta))^{-1} \partial_2 \nabla_1 f(x, \theta)\, .
\]
In the following, we modify slightly the notation of both $A$ and $B$, writing
\begin{align*}
\tilde A(x, \theta) &= \nabla^2_1 f(x, \theta) + \Gamma_\eta(x, \theta) \\
\tilde B(x, \theta) &= - \partial_2 \nabla_1 f(x, \theta)\, .
\end{align*}

With the current hypotheses, following along the proof of Theorem ~\ref{thm:jacob}, we have that $\tilde A$ is $(\alpha + \mu)$ well-conditioned, and $(\gamma + \kappa_\eta)$-Lipschitz in its first argument, and $\tilde B$ is $\beta$-Lipschitz in its first argument and bounded in norm by $R$. The same reasoning yields
\[
\|J(\hat x, \theta) - \partial x^\star (\theta)\| \le \left(\frac{\beta + \kappa_\eta}{\alpha + \mu} + \frac{\gamma R}{(\alpha + \mu)^2}\right) \|\hat x - x^\star(\theta)\|\, .
\]
\end{proof}

\section{The Lasso case}\label{sec:lasso}
Our approach to differentiate the solution of a root equation $F(x,\theta)=0$ is valid as long as the smooth implicit function theorem holds, namely, as long as $F$ is continuously differentiable near a solution $(x_0,\theta_0)$ and $\nabla_1 F(x_0,\theta_0)$ is invertible. While the first assumption is easy to check when $F$ is continuously differentiable \emph{everywhere}, it does not always hold when this is not the case, e.g., when $F$ involves the proximity operator of a non-smooth function. In such cases, one may therefore have to study theoretically the properties of the function $F$ near the solutions $(x(\theta),\theta)$ to justify differentiation using the smooth implicit function theorem. Here, we develop such an analysis to justify the use of our approach to differentiate the solution of a Lasso regression problem with respect to the regularization parameter. We note that the smooth implicit function theorem has already been used for this problem \cite{bertrand_2020_implicit,bertrand_2021_journal}; here we justify why it is a valid approach, even though $F$ itself is not continuously differentiable everywhere. More precisely, we consider the Lasso problem:
\begin{equation}\label{eq:lasso}
\forall \theta\in\RR\,,\quad x^\star(\theta) = \argmin_{x\in\RR^d} \frac{1}{2}\|\dataMatrix x - b\|^2_2 + e^\theta \|x\|_1\,,
\end{equation}
where $\dataMatrix\in\RR^{m\times d}$ and $b\in\RR^m$. Note that we parameterize the regularization parameter as $e^\theta$ to ensure that it remains strictly positive for $\theta\in\RR$, but the analysis below does not depend on this particular choice. This is a typical non-smooth optimization problem where we may want to use a proximal gradient fixed point equation \eqref{eq:proj_grad_fp} to differentiate the solution. In this case we have $f(x,\theta)=(1/2)\|\dataMatrix x-b\|^2_2$, hence $\nabla_1f(x,\theta) = \dataMatrix^\top(\dataMatrix x - b)$, and $g(x,\theta)=e^\theta\|x\|_1$, hence $\prox_{\eta g}(y,\theta) = \ST(y,\eta e^\theta)$ where $\ST:\RR^d \times \RR \rightarrow \RR^d$ is the soft-thresholding operator: $\ST(a,b)_i = \sign(a_i)\times \max(|a_i|-b,0)$. The root equation that characterizes the solution is therefore, for any $\eta>0$:
\begin{equation}
F_\eta(x,\theta) = x - \ST\left(x-\eta \dataMatrix^\top(\dataMatrix x - b), \eta e^\theta\right) = 0\,.
\label{eq:Flasso}
\end{equation}
It is well-known that, under mild assumptions, \eqref{eq:lasso} has a unique solution for any $\theta\in\RR$, and that the solution path $\{x^\star(\theta):\theta\in\RR\}$ is continuous and piecewise linear, with a finite number of non-differentiable points (often called ``kinks'') \cite{Tibshirani2013lasso,Mairal12Complexity}. Since the function $\theta \mapsto x^\star(\theta)$ is not differentiable at the kinks, the smooth implicit function theorem using \eqref{eq:Flasso} is not applicable at those points (as any other methods to compute the Jacobian of $x^\star(\theta)$, such as unrolling). Interestingly, though, the following result shows that, under mild assumptions, the smooth implicit function theorem using \eqref{eq:Flasso} is valid on \emph{all} other points of the solution path, thus justifying the use of our approach to compute the Jacobian of $x^\star(\theta)$ whenever it exists. Note that we state this theorem under some assumption on the matrix $\dataMatrix$ that is sufficient to ensure that the solution to the lasso is unique \cite[Lemma 4]{Tibshirani2013lasso}, but that weaker assumptions could be possible (see proof).
\begin{theorem}\label{thm:lasso}
If the entries of $\dataMatrix$ are drawn from a continuous probability distribution on $\RR^{m\times d}$, then the smooth implicit function theorem using the root equation \eqref{eq:Flasso} holds at any point $(x^\star(\theta),\theta)$ that is not a kink of the solution path, with probability one.
\end{theorem}
\begin{proof}
We first show that $F_\eta$ is continuously differentiable in a neighborhood of $(x^\star(\theta),\theta)$, for any $\theta$ that is not a kink. Since the $\ST$ operator is continuously differentiable everywhere except on the closed set: 
$$
\cS = \left\{(a,b)\in\RR^d\times \RR\,:\,\exists i\in[1,d], |a_i| = b\right\}\,,
$$
it suffices to show from \eqref{eq:Flasso} that $(x^\star(\theta) - \eta \dataMatrix^\top(\dataMatrix x^\star(\theta)-b),\eta e^\theta) \notin \cS$. For that purpose, we  characterize $x^\star(\theta)$ using the subgradient of \eqref{eq:lasso} as follows: there exists $\gamma\in\RR^d$ such that
$$
\dataMatrix^\top\left(b - \dataMatrix x^\star(\theta)\right) = e^\theta \gamma\,,
$$
where
$$
\begin{cases}
\gamma_i = \sign(x^\star(\theta)_i) & \text{ if } x^\star(\theta)_i \neq 0\,,\\
\gamma_i \in [-1,1]& \text{ otherwise.}
\end{cases}
$$
We thus need to show that $(x^\star(\theta) + \eta e^\theta \gamma,\eta e^\theta) \notin \cS$. We first see that, for any $i\in[1,d]$ such that $x^\star(\theta)_i \neq 0$,
    $$
    |x^\star(\theta)_i + \eta e^\theta \gamma_i| = |x^\star(\theta)_i + \eta e^\theta \sign(x^\star(\theta)_i)| > \eta e^\theta\,.
    $$
The case $x^\star(\theta)_i \neq 0$ requires more care, since the property $\gamma_i \in [-1,1]$ is not sufficient to show that $|x^\star(\theta)_i + \eta e^\theta \gamma_i| = \eta e^\theta |\gamma_i|$ is not equal to $\eta e^\theta$: we need to show that, in fact, $|\gamma_i|<1$. For that purpose, let $\cE(\theta) = \left\{i\in[1,d]\,:\,|\gamma_i| = 1 \right\}$. Denoting $\dataMatrix_{\cE(\theta)}$ the matrix made of the columns of $\dataMatrix$ in $\cE(\theta)$, we know that, under the assumptions of Theorem~\ref{thm:lasso}, with probability one the matrix $\dataMatrix_{\cE(\theta)}^\top \dataMatrix_{\cE(\theta)}$ is invertible and the lasso problem has a unique solution given by $x^\star(\theta)_{\cE(\theta)^C} = 0$ and 
\begin{equation}\label{eq:xstarlasso}
x^\star(\theta)_{\cE(\theta)} = (\dataMatrix_{\cE(\theta)}^\top \dataMatrix_{\cE(\theta)})^{-1} \left(\dataMatrix_{\cE(\theta)}^\top b - e^\theta s(\theta)_{\cE(\theta)}\right)\,,
\end{equation}
where $s(\theta) = \sign(\dataMatrix^\top\left(b-\dataMatrix x^\star(\theta) )\right)) \in \{-1,0,1\}^d$ \cite{Tibshirani2013lasso}. Furthermore, we know that $\cE(\theta)$ is constant between two successive kinks \cite{Mairal12Complexity}, so if $x^\star(\theta)$ is not a kink then there is a neighborhood $[\theta_1,\theta_2]$ of $\theta$ such as $\cE(\theta')=\cE(\theta)$ and $s(\theta')_{\cE(\theta')}=s(\theta)_{\cE(\theta)}$, for any $\theta'\in[\theta_1,\theta_2]$. Let us now assume that $\dataMatrix$ is such that for any $\cE \subset [1,d]$ and $s\in\{-1,1\}^{|\cE|}$, $\dataMatrix_\cE^\top \dataMatrix_\cE$ is invertible and $(\dataMatrix_\cE^\top \dataMatrix_\cE)^{-1}s$ has no coordinate equal to zero. This happens with probability one under the assumptions of Theorem~\ref{thm:lasso} since the set of singular matrices is measure zero. Then we see from \eqref{eq:xstarlasso} that, for $\theta'\in[\theta_1 , \theta_2]$ and $i\in\cE(\theta)$, $x^\star(\theta')_i$ is an affine and non-constant function of $e^{\theta'}$. Since in addition $x^\star(\theta_1)_i$ and $x^\star(\theta_2)_i$ are either both nonnegative or nonpositive, then necessarily $x^\star(\theta)_i$ is positive or negative, respectively. In other words, we have shown that $|\gamma_i|=1 \implies x^\star(\theta)_i \neq 0$, or equivalently that $x^\star(\theta)_i = 0 \implies |\gamma_i|<1$. From this we deduce that for any $i\in[1,d]$ such that $x^\star(\theta)_i = 0$,
    $$
    |x^\star(\theta)_i + \eta e^\theta \gamma_i| = \eta e^\theta |\gamma_i| < \eta e^\theta\,.
    $$
This concludes the proof that $(x^\star(\theta) + \eta e^\theta \gamma,\eta e^\theta) \notin \cS$, and therefore that $F_\eta$ is continuously differentiable in a neighborhood of $(x^\star(\theta),\theta)$. The second condition for the smooth implicit theorem to hold, namely, the invertibility of $\nabla_1 F_\eta(x^\star(\theta),\theta)$, is easily obtained by explicit computation \cite[Proposition 1]{bertrand_2020_implicit,bertrand_2021_journal}
\end{proof}

\section{Experimental setup and additional results}
\label{appendix:experimental_setup_and_additional_results}

Our experiments use JAX \cite{jax}, which is Apache2-licensed and
scikit-learn \cite{scikit-learn}, which is BSD-licensed.

\begin{figure}[p]
\begin{subfigure}{.33\textwidth}
  \centering
  \includegraphics[width=\linewidth]{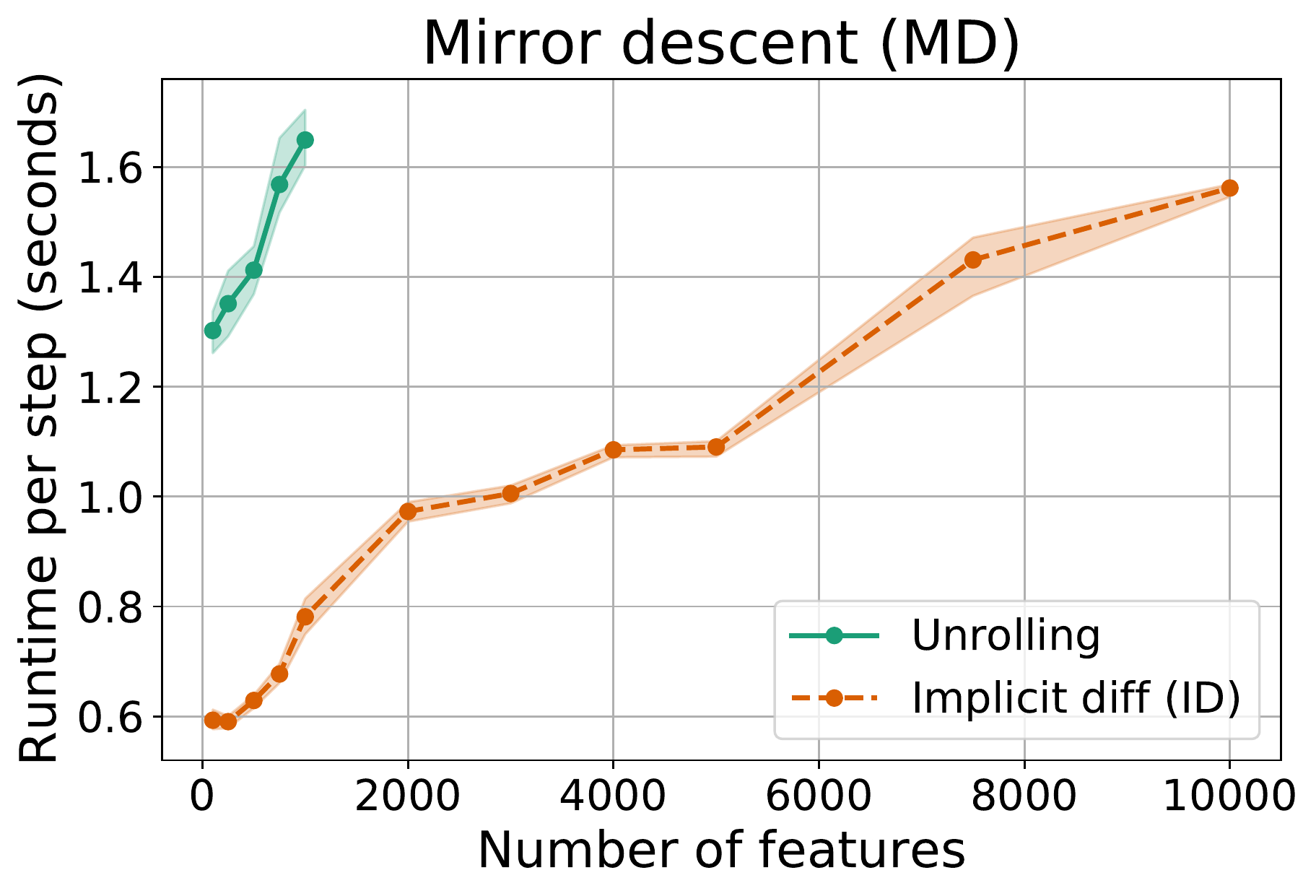}
  \caption{}
  \label{fig:multiclass_svm_runtime_gpu_a}
\end{subfigure}
\begin{subfigure}{.33\textwidth}
  \centering
  \includegraphics[width=\linewidth]{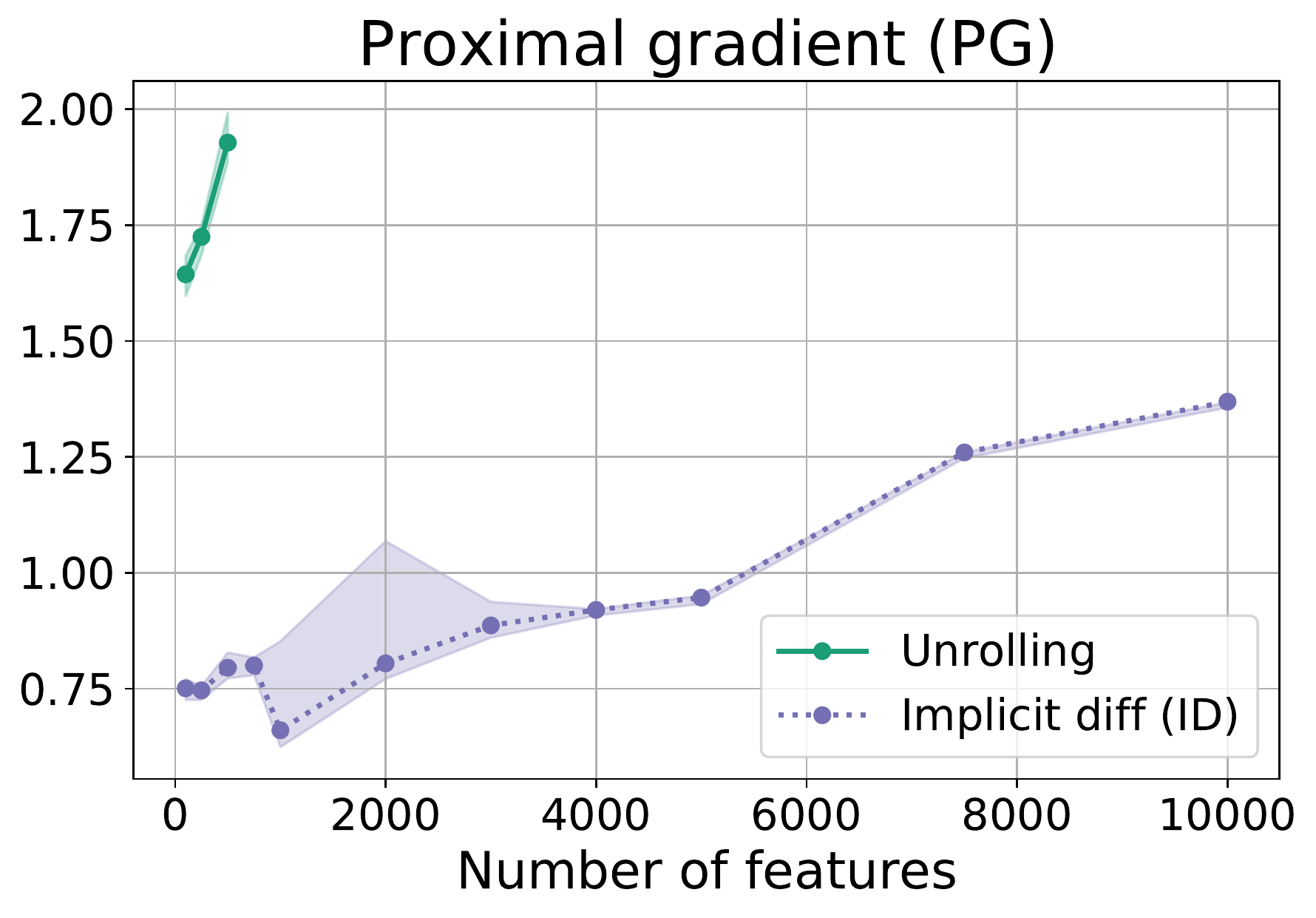}
  \caption{}
  \label{fig:multiclass_svm_runtime_gpu_b}
\end{subfigure}
\begin{subfigure}{.33\textwidth}
  \centering
  \includegraphics[width=\linewidth]{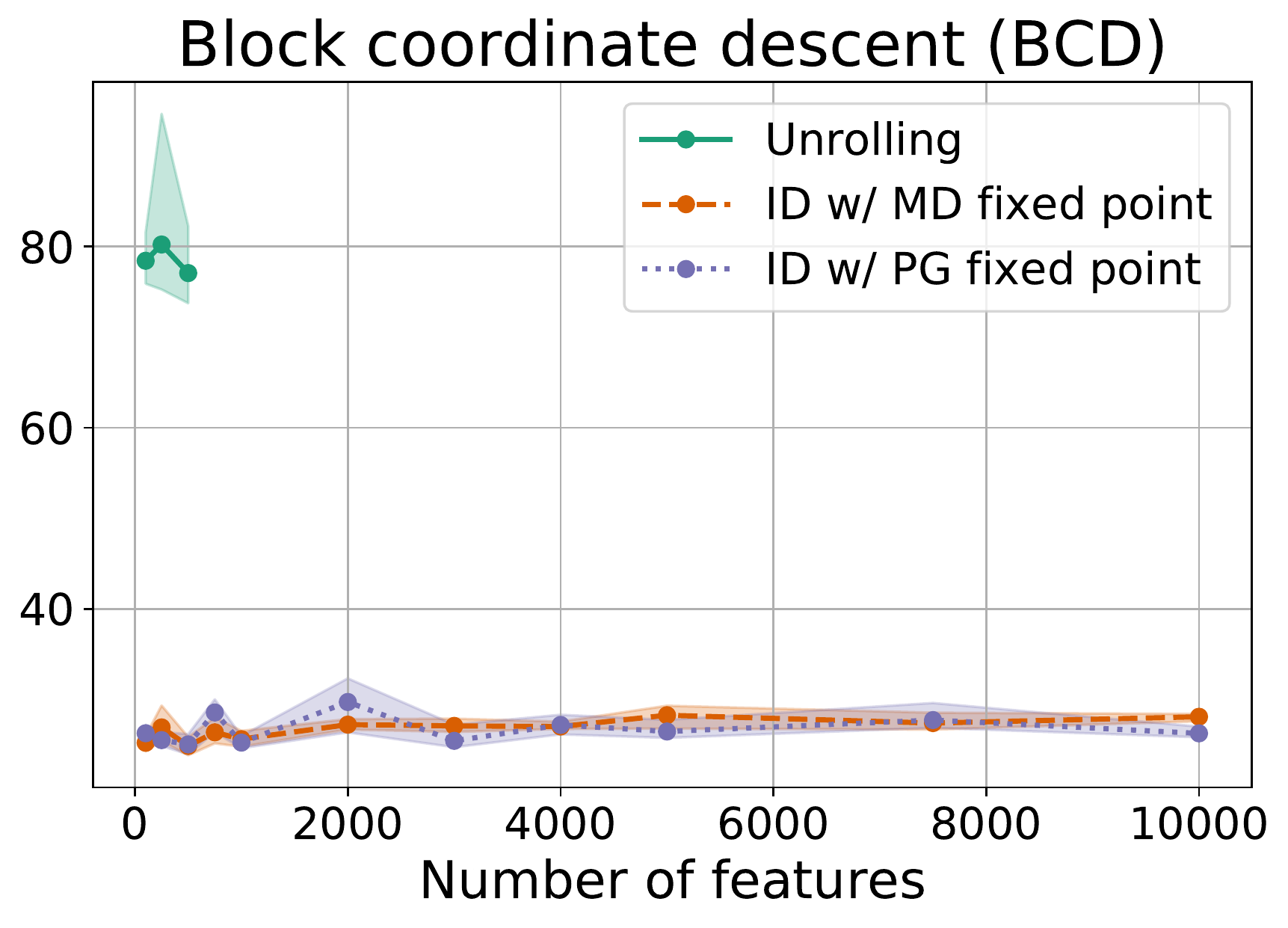}
  \caption{}
  \label{fig:multiclass_svm_runtime_gpu_c}
\end{subfigure}
\caption{GPU runtime comparison of implicit differentiation and unrolling for
    hyperparameter optimization of multiclass SVMs for multiple problem sizes
    (same setting as Figure \ref{fig:multiclass_svm_runtime_cpu}).
    Error bars represent 90\% confidence intervals. Absent data points were due
to out-of-memory errors (16 GB maximum).}
\label{fig:multiclass_svm_runtime_gpu}
\end{figure}

\begin{figure}[p]
\begin{subfigure}{.33\textwidth}
  \centering
  \includegraphics[width=\linewidth]{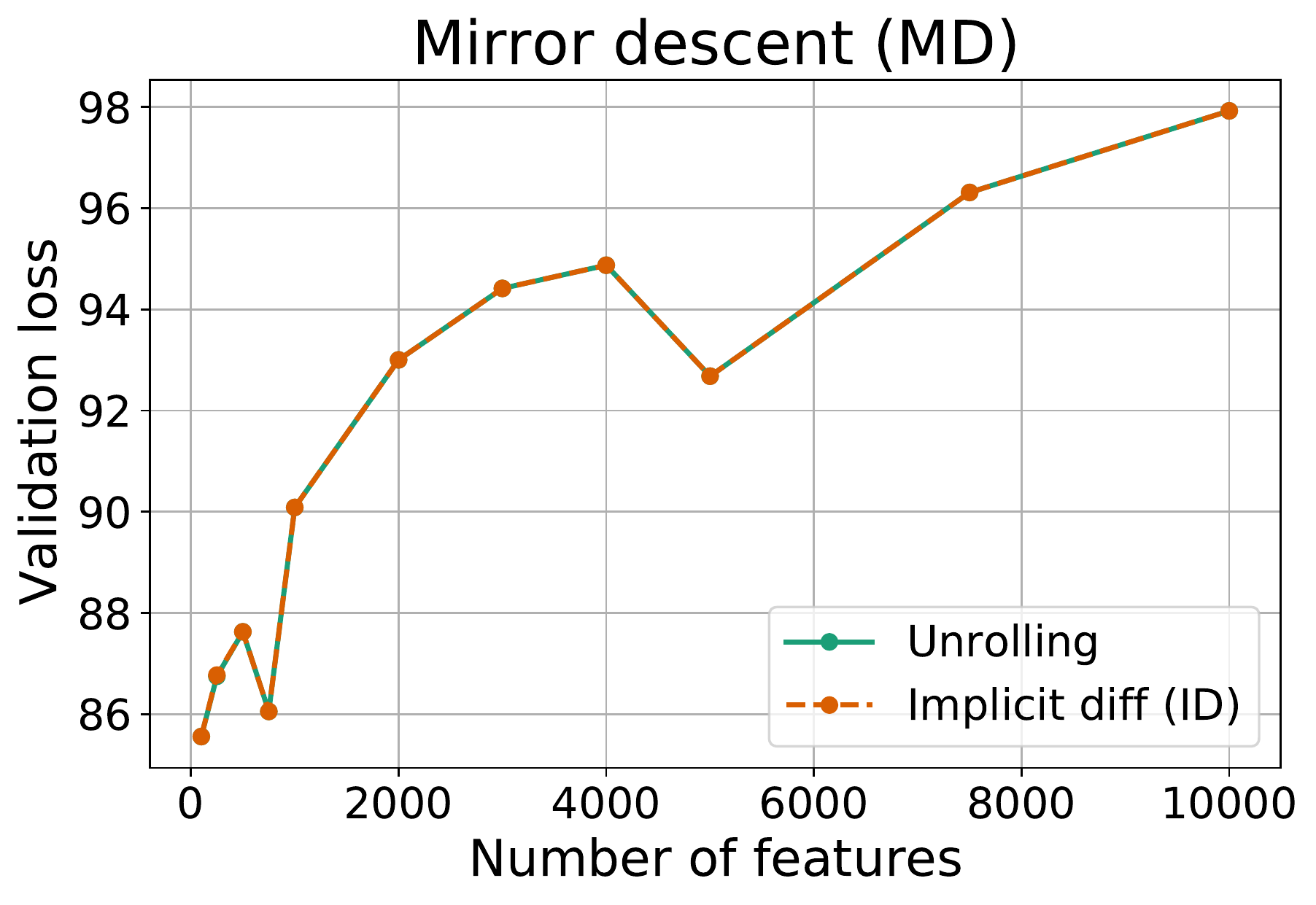}
  \caption{}
  \label{fig:multiclass_svm_loss_cpu_a}
\end{subfigure}
\begin{subfigure}{.33\textwidth}
  \centering
  \includegraphics[width=\linewidth]{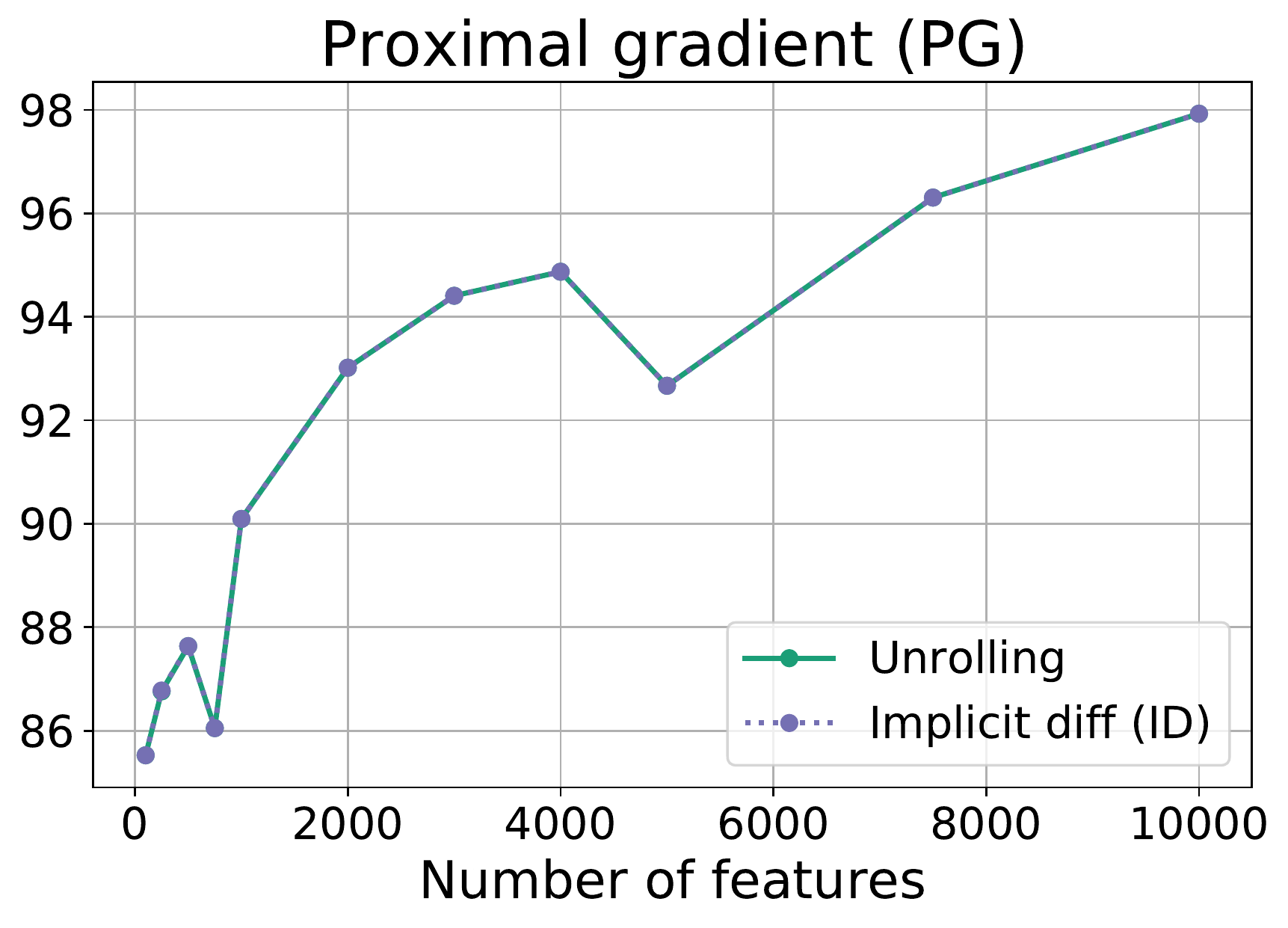}
  \caption{}
  \label{fig:multiclass_svm_loss_cpu_b}
\end{subfigure}
\begin{subfigure}{.33\textwidth}
  \centering
  \includegraphics[width=\linewidth]{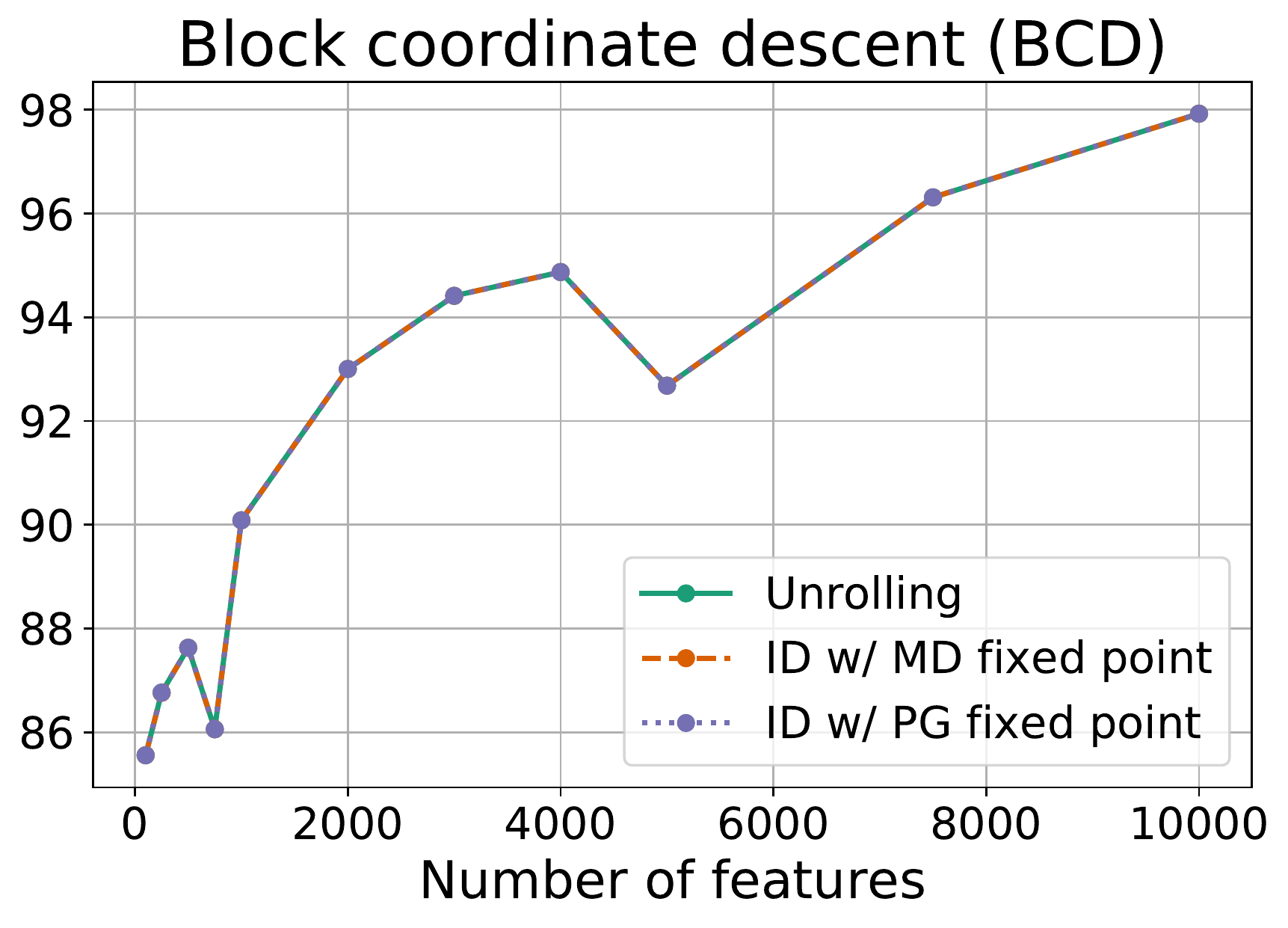}
  \caption{}
  \label{fig:multiclass_svm_loss_cpu_c}
\end{subfigure}
\caption{Value of the outer problem objective function (validation loss) for
    hyperparameter optimization of multiclass SVMs for multiple problem sizes
(same setting as Figure \ref{fig:multiclass_svm_runtime_cpu}).
As can be seen, all methods performed similarly in terms of validation loss.
This confirms that the faster runtimes for implicit differentiation compared to
unrolling shown in
Figure \ref{fig:multiclass_svm_runtime_cpu} (CPU) and Figure
\ref{fig:multiclass_svm_runtime_gpu} (GPU) are not at the cost of worse
validation loss.}
\label{fig:multiclass_svm_loss_cpu}
\end{figure}

\begin{figure}[p]
\begin{subfigure}{.33\textwidth}
  \centering
  \includegraphics[width=\linewidth]{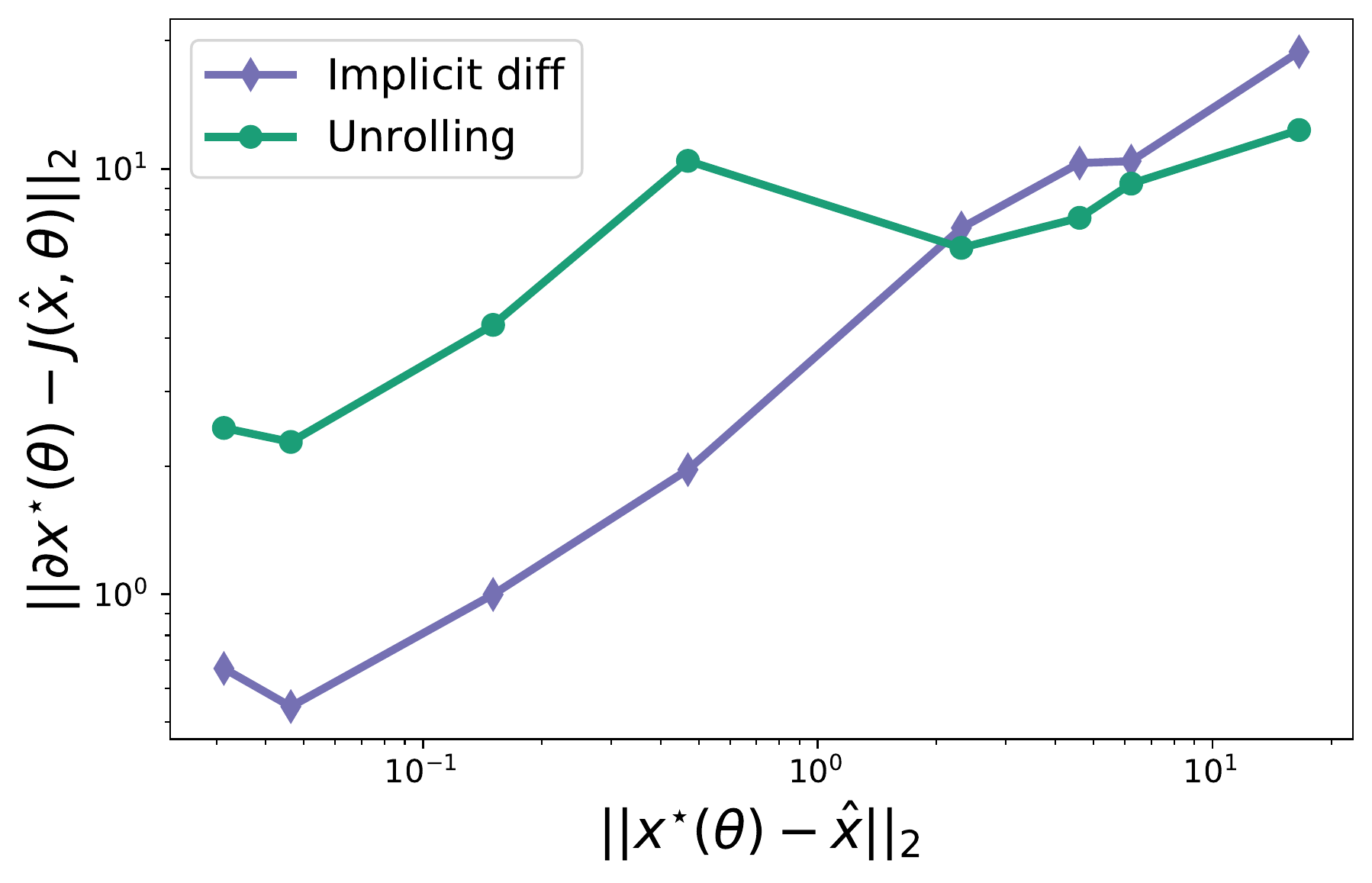}
  \caption{250 features}
  \label{fig:multiclass_svm_loss_jacobian_error_a}
\end{subfigure}
\begin{subfigure}{.33\textwidth}
  \centering
  \includegraphics[width=\linewidth]{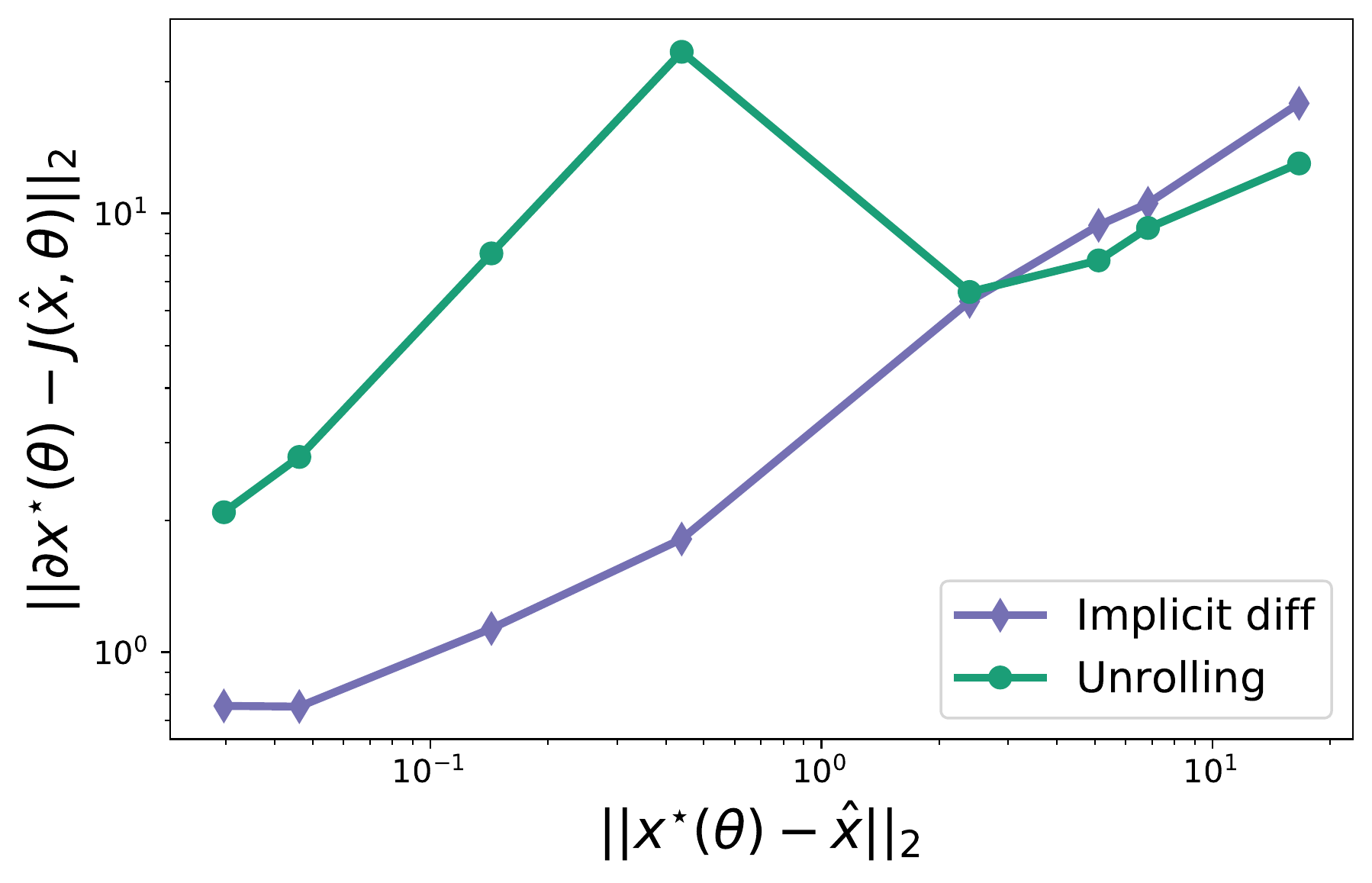}
  \caption{500 features}
  \label{fig:multiclass_svm_jacobian_error_b}
\end{subfigure}
\begin{subfigure}{.33\textwidth}
  \centering
  \includegraphics[width=\linewidth]{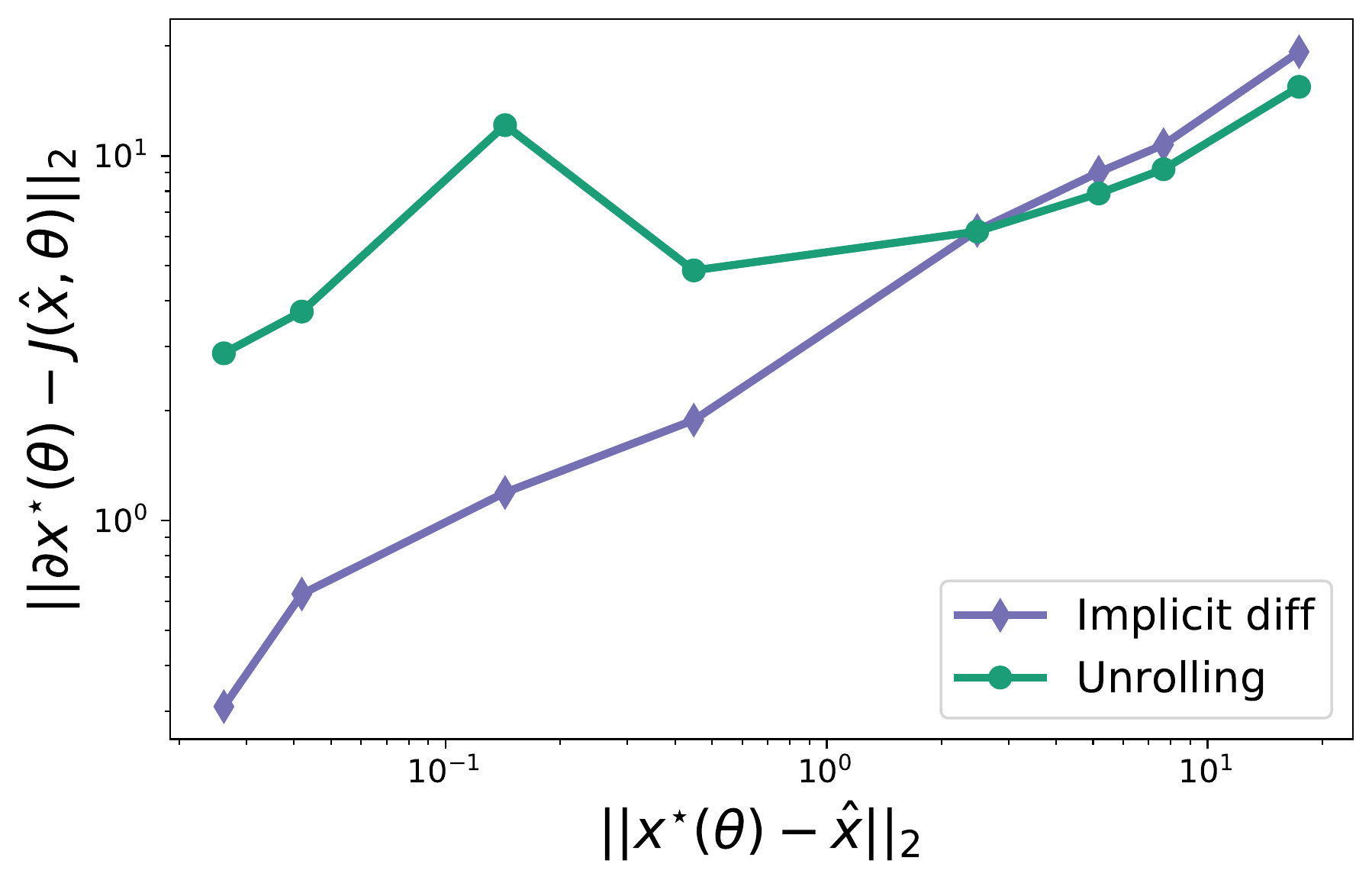}
  \caption{1000 features}
  \label{fig:multiclass_svm_jacobian_error_c}
\end{subfigure}
\caption{Jacobian error $\|\partial x^\star(\theta) - J(\hat x, \theta)\|_2$
(see also Definition \ref{DEF:jac-est}) evaluated with a regularization
parameter of $\theta=1$,
as a function of solution error
$\|x^\star(\theta) - \hat x\|_2$ when varying the number of features, on the
multiclass SVM task (see Appendix
\ref{appendix:multiclass_svm} for a detailed description of the experimental
setup). The ground-truth solution $x^\star(\theta)$ is
computed using the liblinear solver \cite{liblinear} available in scikit-learn
\cite{scikit-learn} with a very low tolerance of $10^{-9}$. 
Unlike in Figure \ref{fig:jac-precision}, which was on ridge regression, the ground-truth Jacobian $\partial
x^\star(\theta)$ cannot be computed in closed form, in the more difficult setting of multiclass SVMs. We therefore use
a finite difference to approximately compute $\partial x^\star(\theta)$.
Our results nevertheless confirm similar trends as in Figure
\ref{fig:jac-precision}.
}
\label{fig:multiclass_svm_jacobian_error}
\end{figure}

\subsection{Hyperparameter optimization of multiclass SVMs}
\label{appendix:multiclass_svm}

\paragraph{Experimental setup.}

Synthetic datasets were generated using \texttt{scikit-learn}'s
\texttt{sklearn.datasets.make\_classification} ~\cite{scikit-learn}, following
a model adapted from~\cite{guyon2003design}. All datasets consist of $m=700$ training samples belonging to $k=5$ distinct classes. To simulate problems of different sizes, the number of features is varied as $p \in \left\{100, 250, 500, 750, 1000, 2000, 3000, 4000, 5000, 7500, 10000\right\}$, with 10\% of features being informative and the rest random noise. In all cases, an additional $m_{\mathrm{val}}=200$ validation samples were generated from the same model to define the outer problem.

For the inner problem, we employed three different solvers: (i) mirror descent, (ii) (accelerated) proximal gradient descent and (iii) block coordinate descent. Hyperparameters for all solvers were individually tuned manually to ensure convergence across the range of problem sizes. For mirror descent, a stepsize of $1.0$ was used for the first $100$ steps, following a inverse square root decay afterwards up to a total of $2500$ steps. For proximal gradient descent, a stepsize of $5 \cdot 10^{-4}$ was used for $2500$ steps. The block coordinate descent solver was run for $500$ iterations. All solvers used the same initialization, namely, $x_{\mathrm{init}}=\frac{1}{k}1_{m \times k}$, which satisfies the dual constraints.

For the outer problem, gradient descent was used with a stepsize of $5 \cdot 10^{-3}$ for the first $100$ steps, following a inverse square root decay afterwards up to a total of $150$ steps.

Conjugate gradient was used to solve the linear systems in implicit differentiation for at most $2500$ iterations.

All results reported pertaining CPU runtimes were obtained using an internal compute cluster. GPU results were obtained using a single NVIDIA P100 GPU with 16GB of memory per dataset. For each dataset size, we report the average runtime of an individual iteration in the outer problem, alongside a 90\% confidence interval estimated from the corresponding $150$ runtime values.

\paragraph{Additional results}

Figure~\ref{fig:multiclass_svm_runtime_gpu} compares the runtime of implicit
differentiation and unrolling on GPU. These results highlight a fundamental
limitation of the unrolling approach in memory-limited systems such as
accelerators, as the inner solver suffered from out-of-memory errors for most
problem sizes ($p \ge 2000$ for mirror descent, $p \ge 750$ for proximal
gradient and block coordinate descent). While it might be possible to ameliorate
this limitation by reducing the maximum number of iterations in the inner
solver, doing so might lead to additional challenges~\cite{wu2018understanding} and require careful tuning.

Figure~\ref{fig:multiclass_svm_loss_cpu} depicts the validation loss (value of the outer problem objective function) at convergence. It shows that all approaches were able to solve the outer problem, with solutions produced by different approaches being qualitatively indistinguishable from each other across the range of problem sizes considered.

Figure \ref{fig:multiclass_svm_jacobian_error} shows the Jacobian error achieved
as a function of the solution error, when varying the number of features.

\subsection{Task-driven dictionary learning}\label{annex:TDDL}
We downloaded from
\url{http://acgt.cs.tau.ac.il/multi_omic_benchmark/download.html} a set of
breast cancer gene expression data together with survival information generated
by the TCGA Research Network (\url{https://www.cancer.gov/tcga}) and processed
as explained by \cite{Rappoport2018Multi}. The gene expression matrix contains the expression value for p=20,531 genes in m=1,212 samples, from which we keep only the primary tumors (m=1,093). From the survival information, we select the patients who survived at least five years after diagnosis ($m_1=200$), and the patients who died before five years ($m_0=99$), resulting in a cohort of $m=299$ patients with gene expression and binary label. Note that non-selected patients are those who are marked as alive but were not followed for 5 years.

To evaluate different binary classification methods on this cohort, we repeated 10 times a random split of the full cohort into a training (60\%), validation (20\%) and test (20\%) sets. For each split and each method, 1) the method is trained with different parameters on the training set, 2) the parameter that maximizes the classification AUC on the validation set is selected, 3) the method is then re-trained on the union of the training and validation sets with the selected parameter, and 4) we measure the AUC of that model on the test set. We then report, for each method, the mean test AUC over the 10 repeats, together with a 95\% confidence interval defined a mean $\pm$ 1.96 $\times$ standard error of the mean.

We used Scikit Learn's implementation of logistic regression regularized by $\ell_1$ (lasso) and $\ell_2$ (ridge) penalty from \texttt{sklearn.linear\_model.LogisticRegression}, and varied the  \texttt{C} regularization parameter over a grid of 10 values: $\{10^{-5}, 10^{-3}, \ldots, 10^4\}$. For the unsupervised dictionary learning experiment method, we estimated a dictionary from the gene expression data in the training and validation sets, using \texttt{sklearn.decomposition.DictionaryLearning(n\_components=10, alpha=2.0)}, which produces sparse codes in  $k=10$ dimensions with roughly $50\%$ nonzero coefficients by minimizing the squared Frobenius reconstruction distance with lasso regularization on the code. We then use \texttt{sklearn.linear\_model.LogisticRegression} to train a logistic regression on the codes, varying the ridge regularization parameter \texttt{C} over a grid of 10 values $\{10^{-1}, 10^{0}, \ldots, 10^8\}$.

Finally, we implemented the task-driven dictionary learning model
(\ref{eq:bilevel_sparse_coding}) with our toolbox, following the pseudo-code in
Figure~\ref{fig:sparse_coding_code}. Like for the unsupervised dictionary
learning experiment, we set the dimension of the codes to $k=10$, and a fixed
elastic net regularization on the inner optimization problem to ensure that the
codes have roughly $50\%$ sparsity. For the outer optimization problem, we solve
an $\ell_2$ regularized ridge regression problem, varying again the ridge
regularization parameter \texttt{C} over a grid of 10 values $\{10^{-1}, 10^{0},
\ldots, 10^8\}$. Because the outer problem is non-convex, we minimize it using
the Adam optimizer~\cite{kingma2014adam} with default parameters.

\subsection{Dataset Distillation}

\paragraph{Experimental setup.}
For the inner problem, we used gradient descent with backtracking line-search, while for the outer problem we used gradient descent with momentum and a fixed step-size. The momentum parameter was set to $0.9$ while the step-size was set to $1$. 

Figure \ref{fig:distillation} was produced after 4000 iterations of the outer
loop on CPU (Intel(R) Xeon(R) Platinum P-8136 CPU @ 2.00GHz), which took 1h55.
Unrolled differentiation took instead 8h:05 ($4$ times more) to run the same
number of iterations. As can be seen in Figure \ref{fig:distillation2}, the output
is the same in both approaches.

\begin{figure}[H]
\includegraphics[width=\textwidth]{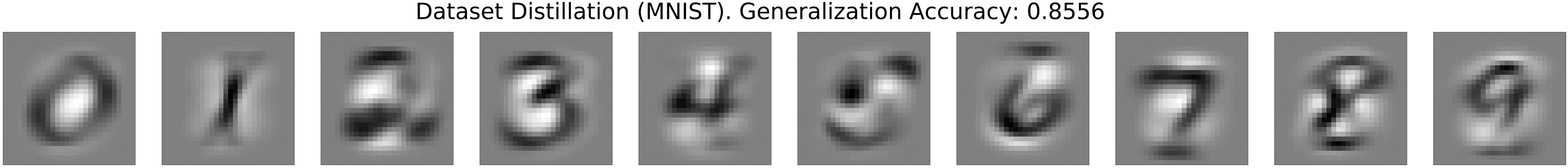}
\caption{Distilled MNIST dataset $\theta \in \RR^{k \times p}$ obtained by
    solving \eqref{eq:bilevel_distillation} through unrolled differentiation. Although there is no qualitative difference, the implicit differentiation approach is 4 times faster.}
\label{fig:distillation2}
\end{figure}

\subsection{Molecular dynamics}

Our experimental setup is adapted from the JAX-MD example notebook available at \url{https://github.com/google/jax-md/blob/master/notebooks/meta_optimization.ipynb}.
 
 We emphasize that calculating the gradient of the total energy objective, $f(x,
 \theta) = \sum_{ij} U(x_{i,j}, \theta)$, with respect to the diameter
 $\theta$ of the smaller particles, $\nabla_1 f(x, \theta)$, does not
 require implicit differentiation or unrolling.
This is because $\nabla_1 f(x, \theta) = 0$ at $x = x^\star(\theta)$:
\begin{equation}
\nabla_\theta f(x^\star(\theta), \theta)
= 
\partial x^\star(\theta)^\top \nabla_1 f(x^\star(\theta), \theta) +
\nabla_2 f(x^\star(\theta), \theta) =
\nabla_2 f(x^\star(\theta), \theta).
\end{equation}
This is known as Danskin's theorem or envelope theorem.  Thus instead, we
consider sensitivities of position $\partial x^\star(\theta)$ directly, which
does require implicit differentiation or unrolling.

Our results comparing implicit and unrolled differentiation for calculating the
sensitivity of position are shown in Figure \ref{fig:jax-md-grad-norm}. We use
BiCGSTAB~\cite{Vorst1992-bicgstab}
to perform the tangent linear solve.  Like in
the original JAX-MD experiment, we use $k=128$ particles in $m=2$ dimensions.

\begin{figure}[H]
    \centering
    \includegraphics[width=0.8\linewidth]{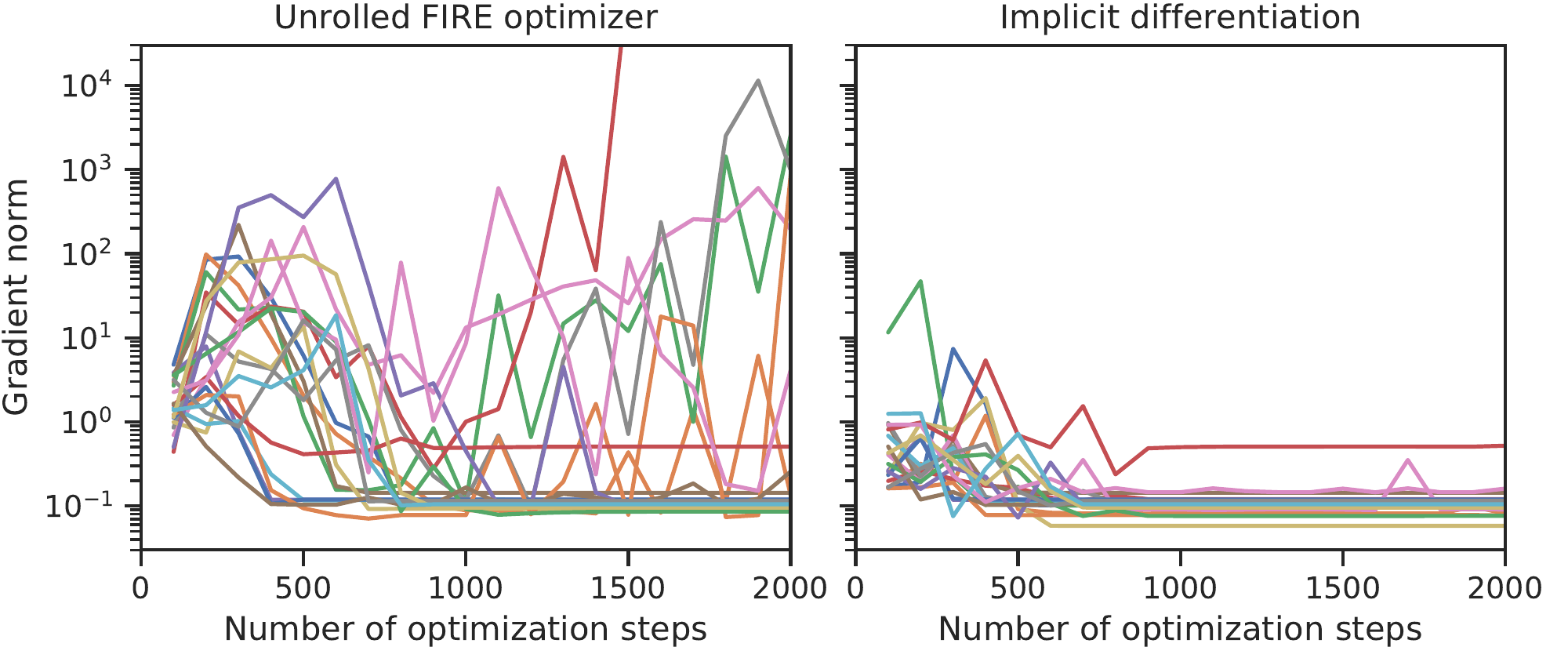}
    \caption{L1 norm of position sensitivities in the molecular dynamics simulations, for 40 different random initial conditions (different colored lines).
    Gradients through the unrolled FIRE optimizer~\cite{structural-relaxation-2006} for many initial conditions do not converge, in contrast to implicit differentiation.}
    \label{fig:jax-md-grad-norm}
\end{figure}


\end{document}